\documentclass[a4paper,11pt]{article}
\usepackage[top=4cm, bottom=4cm, left=3cm, right=3cm]{geometry}
\usepackage[colorlinks,hyperindex]{hyperref}
\usepackage[utf8]{inputenc} 
\usepackage[T1]{fontenc}    
\usepackage{url}            
\usepackage{amsfonts}       
\usepackage{nicefrac}       
\usepackage{microtype}      
\usepackage{bbm, dsfont}

\usepackage{amssymb}
\usepackage{amsmath}
\usepackage{amsthm}
\usepackage{lmodern}
\usepackage{thmtools, thm-restate}
\usepackage{mathrsfs} 
\usepackage[usenames,dvipsnames]{pstricks}
\usepackage{euler}
\usepackage{euscript}
\usepackage{graphicx}
\usepackage{charter}
\usepackage{etoolbox}
\usepackage{mdframed}
\usepackage{floatrow}
\usepackage{caption}
\usepackage{subfig}
\usepackage{color}
\usepackage{thmtools} 
\usepackage{thm-restate}
\usepackage{algorithm,algpseudocode}

\usepackage{hyperref}
\hypersetup{
    colorlinks = true,
    allcolors = {blue}
}
\usepackage{bibunits}

\usepackage{newfloat}

\usepackage[font=footnotesize,skip=0pt]{caption}
\usepackage{booktabs}
\usepackage{cite}

\usepackage[nameinlink,capitalize]{cleveref}

\AfterEndEnvironment{restatable}{\noindent\ignorespaces}

\newcommand{\argmin}{\operatornamewithlimits{argmin}}
\newcommand{\argmax}{\operatornamewithlimits{argmax}}

\def\distance{\mathsf{d}}
\def\dd{\textnormal{d}}

\def\to{\rightarrow}

\def\was{\textnormal{W}}
\def\norm #1{\Vert #1 \Vert}

\def\abs #1{\vert#1\vert}

\def\pspace{\mathscr{P}}

\newcommand{\xx}{\mathcal{X}}
\newcommand{\hh}{\mathcal{H}}
\newcommand{\ff}{\mathcal{F}}
\newcommand{\yy}{\mathcal{Y}}
\newcommand{\hy}{\mathcal{H}_{\mathcal{Y}}}

\newcommand{\X}{X}
\newcommand{\R}{\mathbb{R}}

\def\pspace{\mathcal{P}}

\newcommand{\N}{\mathbb{N}}

\newcommand{\B}{{\mathcal{B}}}

\newcommand{\diag}{\ensuremath{\text{\rm diag}}}

\newcommand{\msf}[1]{{\mathsf{#1}}}
\newcommand{\proj}{{\msf P}}

\newcommand{\eqals}[1]{\begin{align*}#1\end{align*}}
\newcommand{\eqal}[1]{\begin{align}#1\end{align}}

\newcommand{\fhat}{{\widehat f}}
\newcommand{\fstar}{{f^*}}

\newcommand{\gstar}{{g^*}}

\newcommand{\Tla}{T_\la}
\newcommand{\ala}{\alpha_*}
\newcommand{\bla}{\beta_*}

\providecommand{\scal}[2]{\left\langle{#1},{#2}\right\rangle}

\newcommand{\ones}{{\mathbbm{1}}}

\newcommand{\wass}{Wasserstein}

\newcommand{\wreg}{\widetilde{\textnormal{S}}}
\newcommand{\wlam}{\wreg_\lambda}
\newcommand{\wlambda}{\wlam}
\newcommand{\wregla}{\wlam}
\newcommand{\wtilde}{\textnormal{S}}
\newcommand{\wtildela}{\textnormal{S}_\la}

\newcommand{\la}{\lambda}
\newcommand{\E}{\mathcal{E}}

\newcommand{\crop}[1]{{\bar{#1}}}
\newcommand{\loss}{\mathcal{S}}

\renewcommand{\eqref}[1]{Eq.~(\ref{#1})}

\crefname{assumption}{Asm.}{Asm.}
\crefname{equation}{Eq.}{Eqs.}
\crefname{figure}{Fig.}{Figs.}
\crefname{table}{Tab.}{Tabs.}
\crefname{algorithm}{Alg.}{Algs.}
\crefname{section}{Sec.}{Secs.}

\declaretheorem[name=Theorem,refname=Thm.]{theorem}

\declaretheorem[name=Remark]{remark}
\declaretheorem[name=Corollary,refname=Cor.,sibling=theorem]{corollary}
\declaretheorem[name=Definition,refname=Def.,sibling=theorem]{definition}

\declaretheorem[name=Example]{example}

\newcommand{\citep}{\cite}

\title{\sffamily\huge\bf Differential Properties of Sinkhorn Approximation for Learning with Wasserstein Distance}
\author{\small Giulia Luise$^{1}$ \\ {\scriptsize\em g.luise.16@ucl.ac.uk} \and \small Alessandro Rudi$^2$ \\ {\scriptsize\em alessandro.rudi@inria.fr} \\ \and \small Massimiliano Pontil$^{1,3}$ \\ {\scriptsize\em  m.pontil@cs.ucl.ac.uk ~~~~~}   \and \small Carlo Ciliberto$^{1}$ \\ {\scriptsize\em c.ciliberto@ucl.ac.uk} \\ $ $ \\  }

\begin{document}

\maketitle
\begin{abstract}
\noindent Applications of optimal transport have recently gained remarkable attention thanks to the computational advantages of entropic regularization. However, in most situations the  Sinkhorn approximation of the \wass\ distance is replaced by a regularized version that is less accurate but easy to differentiate. In this work we characterize the differential properties of the original Sinkhorn distance, proving that it enjoys the same smoothness as its regularized version and we explicitly provide an efficient algorithm to compute its gradient. We show that this result benefits both theory and applications: on one hand, high order smoothness confers statistical guarantees to learning with \wass\ approximations. On the other hand, the gradient formula allows us to efficiently solve learning and optimization problems in practice. Promising preliminary experiments complement our analysis. 
\end{abstract}


\section{Introduction}

\footnotetext[1]{University College London, WC1E 6BT London, United Kingdom}\footnotetext[2]{INRIA - Sierra Project-team, École Normale Supérieure, Paris, 75012 Paris, France.}\footnotetext[3]{Computational Statistics and Machine Learning - Istituto Italiano di Tecnologia, 16100 Genova, Italy}
Applications of optimal transport have been gaining increasing momentum in machine learning. This success is mainly due to the recent introduction of Sinkhorn approximation \cite{sink,peyre2017computational}, which offers an efficient alternative to the heavy cost of evaluating the \wass\ distance directly. 
The computational advantages have motivated recent applications in optimization and learning over the space of probability distributions, where the \wass\ distance is a natural metric. However, in these settings adopting Sinkhorn approximation requires solving a further optimization problem {\em with respect to} the corresponding distance rather than only evaluating it in a point. This consists in a bi-level problem \cite{dempe2002foundations} for which it is challenging to derive an optimization approach\cite{genevay2018learning}. As a consequence, a regularized version of the Sinkhorn distance is usually considered \cite{courtydomainadap,Frogner2015,pmlr-v51-rolet16,cuturi14, BenamouCCNP15}, for which it is possible to efficiently compute a gradient and thus employ it in first-order optimization methods \cite{cuturi14}. More recently, also efficient automatic differentiation strategies have been proposed \cite{bonneel2016wasserstein}, with applications ranging from dictionary learning \cite{schmitz2018wasserstein} to GANs \cite{genevay2018learning} and discriminat analysis \cite{Flamary2018}. A natural question is whether the easier tractability of this regularization is paid in terms of accuracy. Indeed, while as a direct consequence of \cite{Cominetti1994} it can be shown that the original Sinkhorn approach provides a sharp approximation to the \wass\ distance \cite{Cominetti1994}, the same is not guaranteed for its regularized version. 

In this work we recall both theoretically and empirically that in optimization problems the original Sinkhorn distances is significantly more favorable than its regularized counterpart, which has been indeed noticed to have a tendency to find over-smooth solutions \cite{WassBarySparseSupp}. We take this as a motivation to study the differential properties of the sharp Sinkhorn with the goal of deriving a strategy to address optimization and learning problems over probability distributions. The principal contributions of this work are threefold. Firstly, we show that both Sinkhorn distances are smooth functions. Secondly, we provide an explicit formula to efficiently compute the gradient of the sharp Sinkhorn. As intended, this latter result allows us to adopt this function in applications such as approximating \wass\ barycenters \cite{cuturi14}, which to the best of our knowledge has not been investigated in this setting so far. 

As a third main contribution, we provide a novel sound approach to the challenging problem of {\em learning with Sinkhorn loss}, recently considered in \cite{Frogner2015}. In particular, we leverage the smoothness of the Sinkhorn distance to study the generalization properties of a structured prediction estimator adapted from \cite{CilibertoRR16} to this setting, proving consistency and finite sample bounds. Explicit knowledge of the gradient allows to solve the learning problem in practice. We provide preliminary empirical evidence of the effectiveness of the proposed approach. 

\section{Background: Optimal Transport and Wasserstein Distance}\label{background}
\label{OTtools} 

We provide here a brief overview of the notions used in this work. Given our interest in the computational aspects of optimal transport metrics we refer the reader to \cite{peyre2017computational} for a more in depth introduction to the topic.

Optimal transport theory investigates how to compare probability measures over a domain $\X$. Given a distance function $\distance:\X\times\X\to\R$ between points on $\X$ (e.g. the Euclidean distance on $\X=\R^d$), the goal of optimal transport is to ``translate'' (or lift) it to distances between probability distributions \textit{over} $\X$. 
This allows to equip the space $\pspace(X)$ of probability measures on $\X$ with a metric referred to as {\em \wass\ } distance, which, for any $\mu,\nu \in \pspace(X)$ and $p\geq1$ is defined (see \cite{villani2008optimal}) as
\eqal{\label{def_Wass}
\was_p^p(\mu,\nu)=\inf_{\boldsymbol{\pi}\in \Pi(\mu,\nu)}\int_{X\times X}\distance^p(x,y)\,\dd\boldsymbol{\pi}(x,y),
}
where $\was_p^p$ denotes the $p$-th power of $\was_p$ and where $\Pi(\mu,\nu)$ is the set of probability measures on the product space $X\times X$ whose marginals coincide with $\mu$ and $\nu$; namely 
\begin{equation}\label{TP}
 \Pi(\mu,\nu)=\Big\{~\boldsymbol{\pi}\in\pspace(X\times X) \quad\Big|\quad {\proj_{1}}{\sharp}\boldsymbol{\pi}=\mu,\quad{\proj_{2}}{\sharp}\boldsymbol{\pi}=\nu~\Big\},
\end{equation}
with $\proj_{i}(x_1,x_2)=x_i$ the projection operators for $i=1,2$ and $\proj_i\sharp\boldsymbol{\pi}$ the push-forward of $\boldsymbol{\pi}$ \cite{villani2008optimal}, namely $\proj_1\sharp\boldsymbol\pi = \boldsymbol\pi(\cdot,\X)$ and $\proj_2\sharp\boldsymbol\pi = \boldsymbol\pi(\X,\cdot)$.\\

\noindent {\bf \wass\ Distance on Discrete Measures.} In the following we focus on measures with discrete support. In particular, we consider distributions $\mu,\nu\in\pspace(X)$ that can be written as linear combinations $\mu = \sum_{i=1}^n a_i  \delta_{x_i}$ and $\nu = \sum_{j=1}^m b_j \delta_{y_j}$ of Dirac's deltas centered on a finite number $n$ and $m$ of points $(x_i)_{i=1}^n$ and $(y_j)_{j=1}^m$ in $\X$. In order for $\mu$ and $\nu$ to be probabilities, the vector weights $a = (a_1,\dots,a_n)^\top\in\Delta_n$ and $b = (b_1,\dots,b_m)^\top\in\Delta_m$ must belong respectively to the $n$ and $m$-dimensional simplex, defined as
\eqal{\label{eq:simplex}
  \Delta_n=\left\{ ~p\in\mathbb{R}^{n}_+ ~~ \middle| ~~ p^\top \ones_n = 1 ~\right\}
}
where $\R_+^n$ is the set of vectors $p\in\R^n$ with non-negative entries and $\ones_n\in\R^n$ denotes the vector of all ones, so that $p^\top \ones_n = \begin{matrix}\sum_{i=1}^{n}p_i\end{matrix}$ for any $p\in\R^n$. In this setting, the evaluation of the \wass\ distance corresponds to solving a network flow problem \cite{BTsi} in terms of the weight vectors $a$ and $b$
\begin{equation}\label{Wass_non_reg} 
\was_p^p(\mu,\nu) =\min_{T\in \Pi(a,b)} \langle T,M\rangle
\end{equation}
where $M\in\mathbb{R}^{n\times m}$ is the {\em cost matrix} with entries $M_{ij}=\distance(x_i,y_j)^p$, $\scal{T}{M}$ is the Frobenius product $\textnormal{Tr}(T^\top M)$ and $\Pi(a,b)$ denotes the {\em transportation polytope}
\eqal{
  \Pi(a,b)=\{T\in \mathbb{R}^{n\times m}_{+} \quad \Big| \quad T\mathbbm{1}_m=a,\quad T^{\top}\mathbbm{1}_n=b\Big\},
}
which specializes $\Pi(\mu,\nu)$ in \Cref{TP} to this setting and contains all possible joint probabilities with marginals ``corresponding'' to $a$ and $b$. In the following, with some abuse of notation, we will denote by $\was_p(a,b)$ the \wass\ distance between the two discrete measures $\mu$ and $\nu$ with corresponding weight vectors $a$ and $b$.\\

\noindent {\bf An Efficient Approximation of the \wass\ Distance.} Solving the optimization in \Cref{Wass_non_reg} is computationally very expensive \cite{sink}.  To overcome the issue, the following regularized version of the problem is considered, 
\begin{equation}\label{eq:wass-reg} 
\wreg_\lambda(a,b)~=~\min_{T\in \Pi(a,b)} ~ \scal{T}{M}-\frac{1}{\lambda}h(T) \qquad \textrm{with} \qquad h(T)~=~-\sum_{i,j=1}^{n,m}T_{ij}(\log T_{ij}-1)
\end{equation}
where $\la>0$ is a regularization parameter. Indeed, as observed in \cite{sink}, the addition of the entropy $h$ makes the problem significantly more amenable to computations. In particular, the optimization in \Cref{eq:wass-reg} can be solved efficiently via Sinkhorn's matrix scaling algorithm \cite{sinkhorn1967}. We refer to the function $\wregla$ as the {\em regularized Sinkhorn} distance.

In contrast to the \wass\ distance, the regularized Sinkhorn distance is differentiable (actually smooth, see \Cref{smooth_wtilde}) with respect to both entries $a$ and $b$, hence particularly appealing for practical applications where the goal is to solve a minimization over probability spaces. Indeed, this distance has been recently used with success in settings related to {\em barycenter estimation} \cite{cuturi14,BenamouCCNP15,AltschulerWR17}, supervised learning \cite{Frogner2015} and dictionary learning \cite{pmlr-v51-rolet16}.



\section{Motivation: Better Approximation of the \wass{} Distance}\label{sec:motivations}

The computational benefit provided by the regularized Sinkhorn distance is paid in terms of the approximation with respect to the \wass\ distance. Indeed, the entropic term in \Cref{eq:wass-reg} perturbs the value of the original functional in \Cref{Wass_non_reg} by a term proportional to $1/\la$, leading to potentially very different behaviours of the two functions (\Cref{example:bary-delta} illustrates this effect in practice). In this sense, a natural candidate for a better approximation is
\eqal{\label{wtilde} 
  \wtilde_\la(a,b) ~=~ \scal{~T_\la~}{~M~} \qquad \textrm{with} \qquad T_\la ~=~ \argmin_{T\in\Pi(a,b)}~\scal{~T~}{~M~}-\frac{1}{\lambda}h(T)
}
that corresponds to eliminating the contribution of the entropic regularizer $h(T_\la)$ from $\wreg_\la$ {\em after} the transport plan $T_\la$ has been obtained.  The function $\wtildela$ was originally introduced in \cite{sink} as the Sinkhorn distance, although recent literature on the topic has often adopted this name for the regularized version \Cref{eq:wass-reg}. To avoid confusion, in the following we will refer to $\wtildela$ as the {\em sharp Sinkhorn} distance. Note that we will interchangeably use the notations $\wtilde_\la(a,b)$ and $ \wtilde_\la(\mu,\nu)$ where clear from the context. 

The function $\wtildela$ defined in \Cref{wtilde}  is nonnegative and safisfies the triangular inequality. However, $\wtildela(a,a)\neq 0$, and hence $\wtildela$ is not -strictly speaking- a distance on $\Delta_n$. As shown in \cite{sink}(Thm. 1), it sufficies to multiply  $\wtildela(a,b)$ by $\mathbf{1}_{a\neq b}$  to recover an actual distance which satisfies all the axioms. Despite this fact, with some sloppiness we will refer to $\wtildela$ itself as Sinkhorn \textit{distance}.

As the intuition suggests, the absence of the entropic term $h(T_\la)$ is reflected in a faster rate at approximating the \wass\ distance. The following result makes this point precise. 


\begin{restatable}{proposition}{convtowass}\label{convergence_to_wass} Let $\lambda>0$. For any pair of discrete measures $\mu,\nu\in\pspace(X)$ with respective weights $a\in\Delta_n$ and $b\in\Delta_m$, we have
\eqal{\label{eq:convergence-to-wass}
\big|~\wtilde_\lambda(\mu,\nu)-\was(\mu, \nu)~\big|\leq c_1\, \textnormal{e}^{-\lambda} \qquad\qquad \big|~\wreg_\la(\mu,\nu)-\was(\mu,\nu)~\big|\leq c_2\lambda^{-1},
}
where $c_1,c_2$ are constants independent of $\lambda$, depending on the support of $\mu$ and $\nu$. 
\end{restatable}
The proof of \Cref{convergence_to_wass} is a direct consequence of the result in \cite{Cominetti1994} (Prop. $5.1$), which proves the exponential convergence of  $\Tla$ in \Cref{wtilde} to the optimal plan of $\was$ with maximum entropy, namely 
\eqal{
\Tla \rightarrow T^*=\argmax \Big\{~ h(T) ~\Big| ~T\in\Pi(a,b) \quad \langle T,M\rangle=\was(\mu,\nu) ~\Big\}
 }
as $\la\to+\infty$. While the sharp Sinkhorn distance $\wtildela$ preserves the rate of converge of $\Tla$, the extra term $\lambda^{-1}h(\Tla)$ in the definition of the regularized Sinkhorn distance $\wlambda$ causes the slower rate. In particular, \Cref{eq:convergence-to-wass} (Right) is a direct consequence of \cite{CuturiP16} (Prop. 2.1). In the appendix we provide more context on the derivation of the two inequalities.

\Cref{convergence_to_wass} suggests that, given a fixed $\la$, the sharp Sinkhorn distance 
can offer a more accurate approximation of the \wass\ distance. This intuition is further supported by \Cref{example:bary-delta} where we compare the behaviour of the two approximations on the problem of finding an optimal transport barycenter of probability distributions.\\

\noindent {\bf \wass\ Barycenters}. Finding the barycenter of a set of discrete probability measures ${\cal D} = (\nu_i)_{i=1}^\ell$ is a challenging problem in applied optimal transport settings \cite{cuturi14}. The {\em \wass\ barycenter} is defined as
\eqal{\label{eq:bary-functional}
  \mu_{\was}^* ~=~ \argmin_{\mu} ~ {\B_{\was}}(\mu,{\cal D}), \qquad\qquad \B_{\was}(\mu,{\cal{D}}) ~=~ \sum_{i=1}^\ell~ \alpha_i~ \was(\mu,\nu_i),
}
namely the point $\mu_{\was}^*$ minimizing the weighted average distance between all distributions in the set ${\cal D}$, with $\alpha_i$ scalar weights. Finding the \wass\ barycenter is computationally very expensive and the typical approach is to approximate it with the barycenter $\tilde\mu_\la^*$, obtained by substituting the \wass\ distance $\was$ with  the regularized Sinkhorn distance $\wreg_\la$ in the the objective functional of \Cref{eq:bary-functional}. However, in light of the result in \Cref{convergence_to_wass}, it is natural to ask whether the corresponding baricenter $\mu_\la^*$ of the sharp Sinkhorn distance $\wtilde_\lambda$ could provide a better estimate of the \wass\ one. While we defer an empirical comparison of the two barycenters to \Cref{Exp}, here we consider a simple scenario in which the sharp Sinkhorn can be proved to be a significantly better approximation of the \wass\ distance.

\begin{example}[Barycenter of two Deltas]\label{example:bary-delta}
We consider the problem of estimating the barycenter of two Dirac's deltas $\mu_1 = \delta_{z}, \mu_2 = \delta_{y}$ centered at $z=0$ and $y = n$  with $z,y\in\R$ and $n$ an even integer. Let $\X = \{x_0,\dots,x_n\}\subset\R$ be the set of all integers between $0$ and $n$ and $M$ the cost matrix with squared Euclidean distances. Assuming uniform weights $\alpha_1 = \alpha_2$, it is well-known that the \wass\ barycenter is the delta centered on the euclidean mean of $z$ and $y$, $\mu_{\was}^* = \delta_{\frac{z+y}{2}}$. A direct calculation (see Appendix \ref{Apsec:barydelte}) shows instead that the \emph{regularized} Sinkhorn barycenter $\tilde{\mu}_{\la}^* = \sum_{i=0}^n a_i \delta_{x_i}$ tends to spread the mass across all $x_i\in\X$, accordingly to the amount of regularization,
\eqal{
  a_i \propto \textnormal{e}^{-\la((z-x_i)^2+(y- x_i)^2)/2} \qquad i=0,\dots,n,
}
behaving similarly to a (discretized) Gaussian with standard deviation of the same order of the regularization $\lambda^{-1}$. On the contrary, 
 the \emph{sharp} Sinkhorn barycenter equals the \wass\ one, namely $\mu_\la^* = \mu_{\was}^*$ {\em for every} $\la>0$. An example of this behavior is reported in \Cref{Figbary}.
\end{example}

\begin{figure}[t]
 \centering
  \includegraphics[width=0.6\columnwidth]{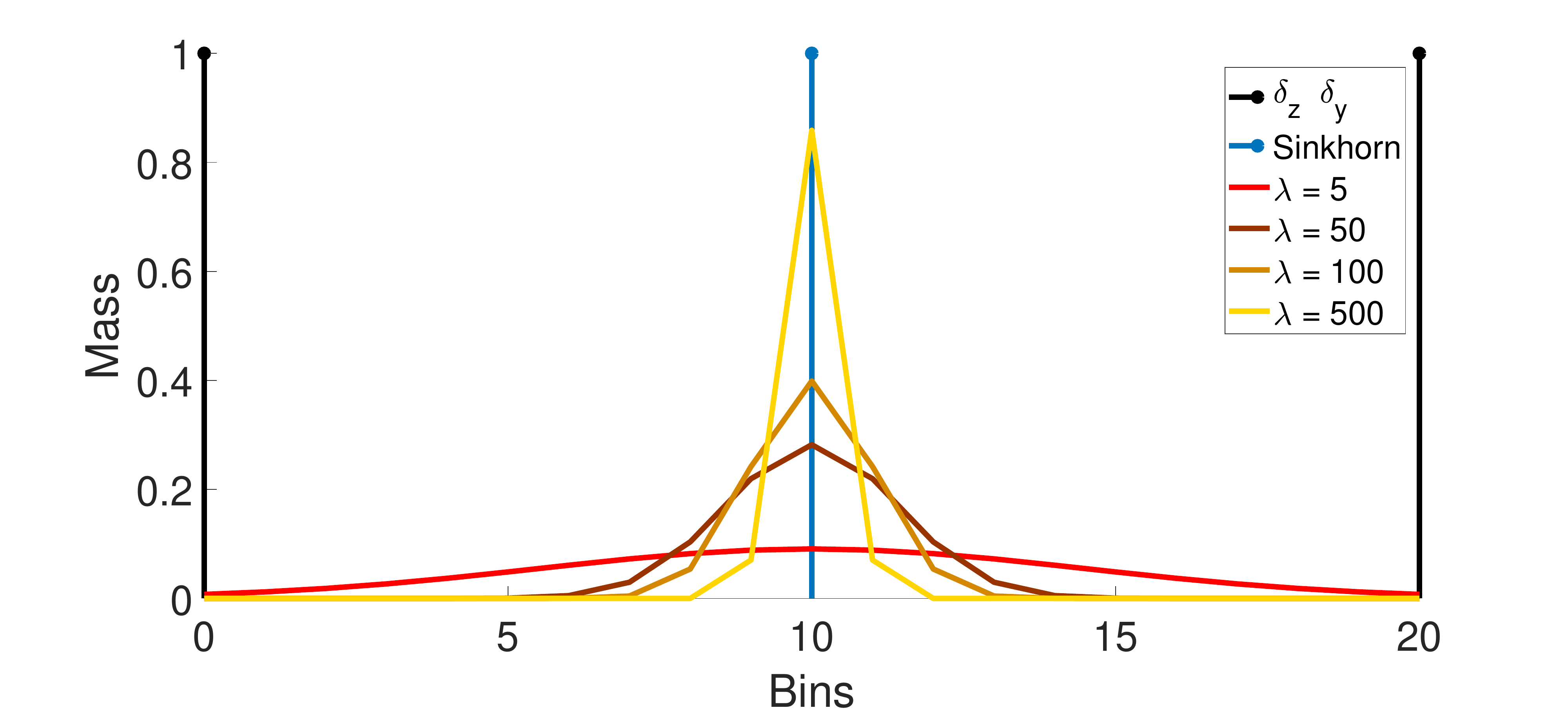}
  \caption{Comparison of the sharp (Blue) and regularized (Orange) barycenters of two Dirac's deltas (Black) centered in 0 and 20 for different values of $\la$.}
  \label{Figbary}
\end{figure}

\noindent {\bf Main Challenges of the Sharp Sinkhorn.} The example above, together with \Cref{convergence_to_wass}, provides a strong argument in support of adopting the sharp Sinkhorn distance over its regularized version. 
However, while the gradient of the regularized Sinkhorn distance can be easily computed  (see \cite{cuturi14} or \Cref{sec:differential-properties}) and therefore it is possible to address optimization problems such as the barycenter in \Cref{eq:bary-functional} with first-order methods (e.g. gradient descent), an explicit form for the gradient of the sharp Sinkhorn distance has not been considered. Also, approaches based on automatic differentiation have been recently adopted to compute the gradient of a variant of $\wtildela$, where the plan $\Tla$ is the one obtained after a fixed number $L$ of iterations \cite{genevay2018learning,schmitz2018wasserstein,genevay2018learning,Flamary2018}. These methods have been observed to be both computationally very efficient and also very effective in practice on a number of machine learning applications. However, in this work we are interested in investigating the analytic properties of the gradient of the sharp Sinkhorn distance, for which we provide an explicit algorithm in the following.

\section{Differential Properties of Sinkhorn Distances}\label{sec:differential-properties}

In this section we present two main results of this work, namely a proof of the smoothness of the two Sinkhorn distances introduced above, and the explicit derivation of a formula for the gradient $\wtildela$. These results will be key to employ the sharp Sinkhorn distance in practical applications. The results are obtained leveraging the Implicit Function Theorem \cite{edwards2012advanced} via a proof technique analogous to that in \cite{bengio2000gradient,chapelle2002choosing,Flamary2018} which we outline in this section and discuss in detail in the appendix.


\begin{restatable}{theorem}{cinftysmooth}\label{smooth_wtilde}
For any $\la>0$, the Sinkhorn distances $\wregla$ and $\wtilde_\lambda: \Delta_n\times \Delta_n \rightarrow \mathbb{R}$ are $\textnormal{C}^\infty$ in the interior of their domain.
\end{restatable}

\noindent \Cref{smooth_wtilde} guarantees both Sinkhorn distances to be infinitely differentiable. In \Cref{learning} this result will allow us to derive an estimator for supervised learning with Sinkhorn loss and characterize its corresponding statistical properties (i.e. universal consistency and learning rates). The proof of \Cref{smooth_wtilde} is instrumental to derive a formula for the gradient of $\wtildela$. We discuss here its main elements and steps while referring to the supplementary material for the complete proof. 

\begin{proof}[Sketch of the proof] 

The proof of \Cref{smooth_wtilde} hinges on the characterization of the (Lagrangian) dual problem of the regularized Sinkhorn distance in \Cref{eq:wass-reg}. This can be formulated (see e.g. \cite{sink}) as
\eqal{\label{functL}
  \max_{\alpha,\beta} ~\mathcal{L}_{a,b}(\alpha,\beta), \qquad \mathcal{L}_{a,b}(\alpha,\beta) = \alpha^\top a + \beta^\top b - \frac{1}{\la} \sum_{i,j=1}^{n,m} \textnormal{e}^{-\la(M_{ij} - \alpha_i - \beta_j)} 
}
with dual variables $\alpha\in\R^n$ and $\beta\in\R^m$. 

\noindent By Sinkhorn's scaling theorem \cite{sinkhorn1967}, the optimal primal solution $\Tla$ in \Cref{wtilde} can be obtained from the dual solution $(\ala,\bla)$ of \Cref{functL} as
\eqal{\label{eq:primal-dual-solution}
  \Tla = \diag(\textnormal{e}^{\la \ala})~\textnormal{e}^{-\la M}~ \diag(\textnormal{e}^{\la \bla}),
}
where for any ${\mathsf v}\in\R^n$, the vector $\textnormal{e}^\mathsf{v}\in\R^n$ denotes the element-wise exponentiation of $\msf v$ (analogously for matrices) and $\diag(\mathsf{v})\in\R^{n \times n}$ is the diagonal matrix with diagonal corresponding to $\mathsf{v}$. \\

\noindent Since both Sinkhorn distances are smooth functions of $\Tla$, it is sufficient to show that $\Tla(a,b)$ itself is smooth as a function of $a$ and $b$.  Given the characterization of \Cref{eq:primal-dual-solution} in terms of the dual solution, this amounts to prove that $\ala(a,b)$ and $\bla(a,b)$ are smooth with respect to $(a,b)$, which is the most technical step of the proof and can be shown leveraging the Implicit Function Theorem \cite{edwards2012advanced}. Indeed, the dual variables $\ala(a,b)$ and $\bla(a,b)$ are obtained as argmax of the strictly convex function $\mathcal{L}$ \Cref{functL}. The argmax corresponds to a stationary point of the gradient, which means $\nabla_{\alpha,\beta}\mathcal{L}(\ala,\bla)=0$. The last part of the proof consists in applying the Implicit Function Theorem to the function $\nabla_{\alpha,\beta}\mathcal{L}$. Note that the theorem can be applied thanks to the strict convexity of $\mathcal{L}$, which guarantess that the Jacobian of $\nabla_{\alpha,\beta}\mathcal{L}$, which is the Hessian of $\mathcal{L}$ is invertible. All the details are discussed at length in the Appendix.
\end{proof}


\noindent {\bf The gradient of Sinkhorn distances}. We now discuss how to derive the gradient of Sinkhorn distances with respect to one of the two variables. In both cases, the dual problem introduced in \Cref{functL} plays a fundamental role. In particular, as pointed out in \cite{cuturi14}, the gradient of the regularized Sinkhorn distance can be obtained directly from the dual solution as $\nabla_a\wregla (a,b) = \ala(a,b)$,
for any $a\in\R^n$ and $b\in\R^m$. This characterization is possible because of well-known properties of primal and dual optimization problems \cite{BTsi}.

\noindent The sharp Sinkhorn distance does not have a formulation in terms of a dual problem and therefore a similar argument does not apply. Nevertheless, we show here that it is still possible to obtain its gradient in closed form in terms of the dual solution. 

\begin{restatable}{theorem}{gradient}\label{gradientwtilde}
Let $M\in\R^{n \times m}$ be a cost matrix, $a\in\Delta_n$, $b\in\Delta_m$ and $\la>0$. Let $\mathcal{L}_{a,b}(\alpha,\beta)$ be defined as in \eqref{functL}, with argmax in $(\ala,\bla)$. Let $\Tla$ be defined as in \Cref{eq:primal-dual-solution}. Then,
\begin{equation}\label{gradientSink}
  \nabla_a \wtilde_\lambda (a,b)=\proj_{\textnormal{T} \Delta_n}\big(A~ L \mathbbm{1}_m + B ~\crop{ L}^\top \mathbbm{1}_n\big)
\end{equation}
where $L=\Tla \odot M\in\R^{n \times m}$ is the entry-wise multiplication between $\Tla$ and $M$ and $\crop{ L}\in\R^{{n}\times m-1}$ corresponds to $L$ with the last column removed. The terms $A\in\R^{n \times n}$ and $B\in\R^{n \times {m-1}}$ are
\begin{equation}\label{eq:hessian-a-and-b}
  [A ~ B]=-\lambda ~D~~\big[~\nabla^2_{(\alpha,\beta)}\mathcal{L}_{a,b}(\ala,\bla)~\big]^{-1},
\end{equation}
with $D=[\textnormal{I} ~\mathbf{0}]$ the matrix concatenating the $n\times n$ identity matrix $\textnormal I$ and the matrix $\mathbf{0}\in\R^{n \times m-1}$ with all entries equal to zero. The operator $\proj_{\textnormal{T} \Delta_n}$ denotes the projection onto the tangent plane $\textnormal{T}\Delta_n = \{ x\in \mathbb{R}^n:\,\,\sum_{i=1}^n x_i=0\}$ to the simplex $\Delta_n$.
\end{restatable}

  \begin{algorithm}[t]
  \newcommand{\tn}{\textnormal}
    \caption{Computation of $\nabla_a \wtildela(a,b)$}\label{alggradient}
    \begin{algorithmic}
      \State ~
      \State {\bfseries Input:} $a \in \Delta_n, \, b\in \Delta_m$, cost matrix $M\in\mathbb{R}_{+}^{n,m}$, $\lambda>0$.
      \State~
      \State~~ $T ~=~ ${\sc Sinkhorn}$(a,b,M,\la),$ ~~~ $\bar T ~=~ T_{1:n,1:(m-1)}$ 
      \State~~ $L ~=~ T \odot M$,~~~~ ~~~~~~~$\bar L ~=~ L_{1:n,1:(m-1)}$
      \State~~ $D_1 ~=~ \diag(T\ones_m)$,~~~$D_2 ~=~ \diag(\bar T^\top \ones_n)^{-1}$
     \State~~ $ H ~=~ D_1 - \bar T D_2 \bar T^\top$,
     \State~~ $\tn f ~=~ -L \ones_m + \bar T D_2 \bar L^\top \ones_{n}$
      
      \State~~ $\tn g ~=~ H^{-1}~\tn f$
      \State ~

      \State {\bfseries Return:} $\tn g - \ones_n (\tn g^\top \ones_n)$
    \end{algorithmic}
  \end{algorithm}

\noindent The proof of \Cref{gradientwtilde} can be found in the supplementary material (Sec. \ref{sec:appdiff}). The result is obtained by first noting that the gradient of $\wtildela$ is characterized (via the chain rule) in terms of the the gradients $\nabla_a\ala(a,b)$ and $\nabla_a\bla(a,b)$ of the dual solutions. The main technical step of the proof is to show that these gradients correspond respectively to the terms $A$ and $B$ defined in \Cref{eq:hessian-a-and-b}.

To obtain the gradient of $\wtildela$ in practice, it is necessary to compute the Hessian $\nabla^2_{(\alpha,\beta)}\mathcal{L}_{a,b}(\ala,\bla)$ of the dual functional. A direct calculation shows that this corresponds to the matrix 
\eqal{\label{eq:hessian}
  \nabla^2_{(\alpha,\beta)}\mathcal{L}(\ala,\bla) = \left[~\begin{matrix} \diag(a) & \bar \Tla \\ \bar \Tla^\top & \diag(\bar b)\end{matrix}~\right],
}
where $\bar \Tla$ (equivalently $\bar b$) corresponds to $\Tla$ (respectively $b$) with the last column (element) removed. See the supplementary material (Sec. \ref{sec:appdiff}) for the details of this derivation.

From the discussion above, it follows that the gradient of $\wtildela$ can be obtained in closed form in terms of the transport plan $\Tla$. \Cref{alggradient} reports an efficient approach to perform this operation. The algorithm can be derived by simple algebraic manipulation of \Cref{gradientSink}, given the characterization of the Hessian in \Cref{eq:hessian}. We refer to the supplementary material for the detailed derivation of the algorithm.\\

\noindent{\bf Barycenters with the sharp Sinkhorn.} Using \Cref{alggradient} we can now apply the accelerated gradient descent approach proposed in \cite{cuturi14}  
to find barycenters with respect to the sharp Sinkhorn distance. \Cref{baryell} reports a qualitative experiment inspired by the one in \cite{cuturi14}, with the goal of comparing the two Sinkhorn barycenters. We considered $30$ images of random nested ellipses on a $50 \times 50$ grid. We interpret each image as a distribution with support on pixels. The cost matrix is given by the squared Euclidean distances between pixels. \Cref{baryell} shows some examples images in the dataset and the corresponding barycenters of the two Sinkhorn distances. 
While the barycenter $\tilde\mu_\la^*$ of $\wlambda$ suffers a blurry effect, the $\wtildela$ barycenter $\mu_\la^*$ is very sharp, suggesting a better estimate of the ideal one. 

\begin{figure}[t]
 \centering
  \includegraphics[height=0.3\columnwidth]{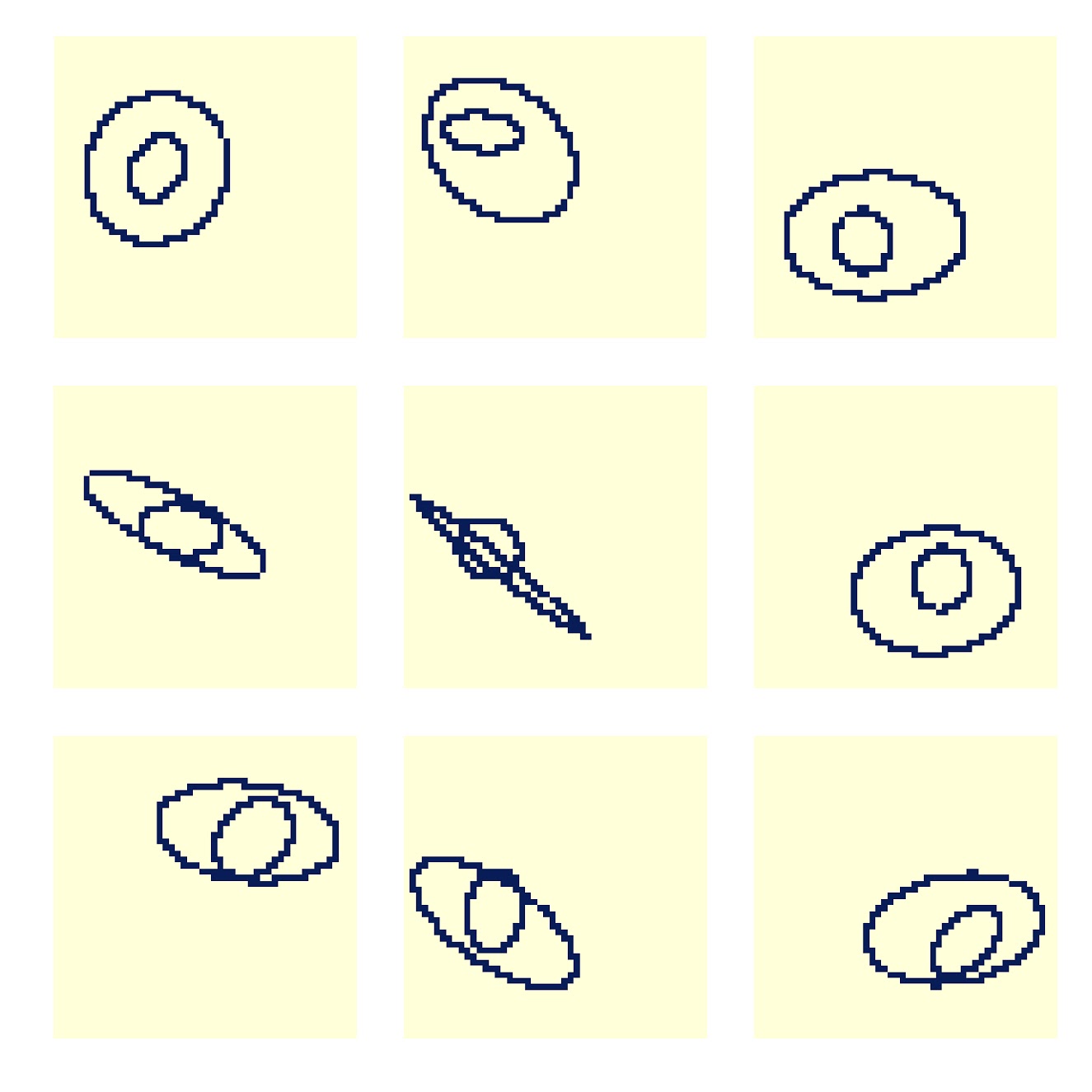}\quad
    \includegraphics[height=0.3\columnwidth]{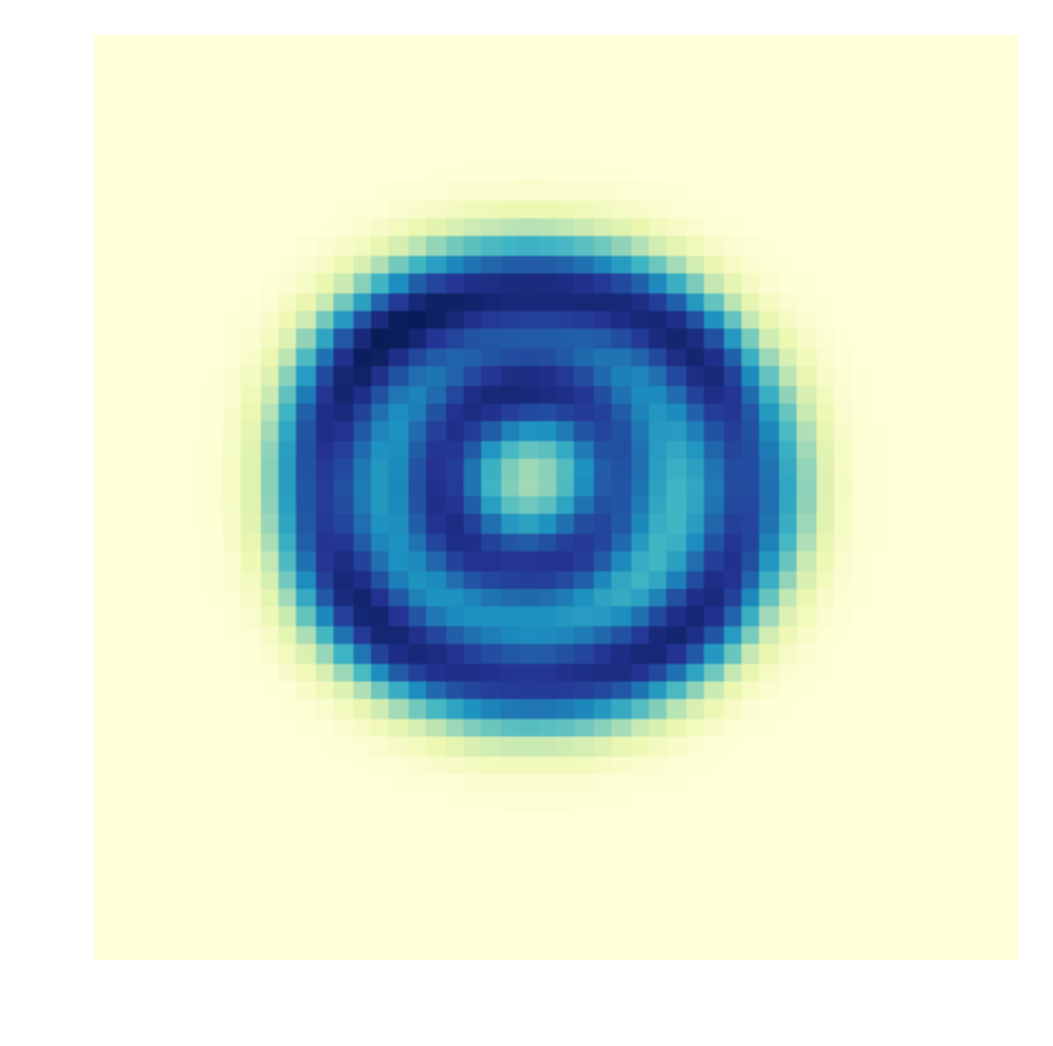} \quad
    \includegraphics[height=0.3\columnwidth]{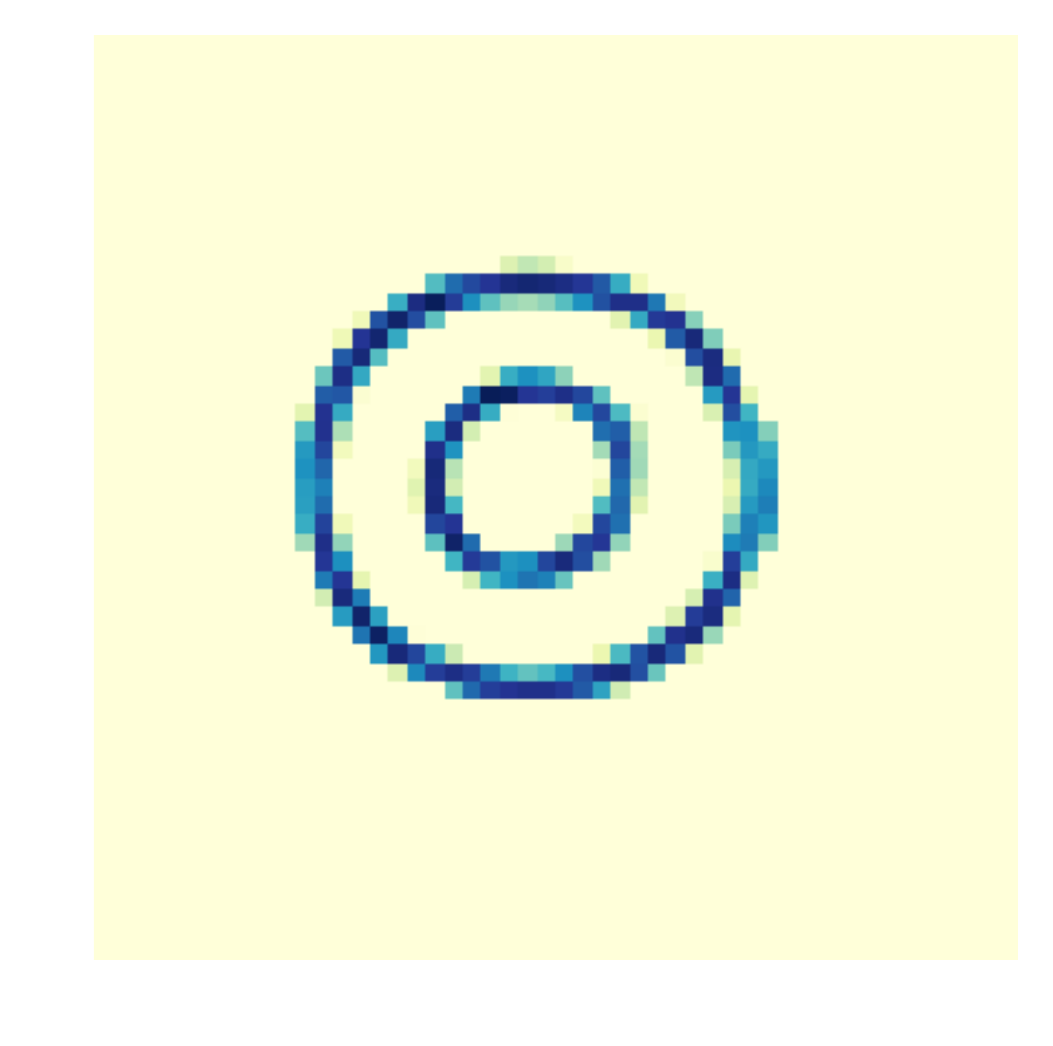}
  \caption{Nested Ellipses: (Left) Sample input data. (Middle) Regularized (Right) sharp Sinkhorn barycenters.}
  \label{baryell}
\end{figure} 

We conclude this section with a computational consideration on the two methods. 

\begin{remark}[Computations]\label{rem:computation-comparison}
We compare the computational complexity of evaluating the gradients of $\wregla$ and $\wtildela$. 
Both gradients rely on the solution of the Sinkhorn problem in \Cref{eq:wass-reg}, which requires $O(n m \epsilon^{-2} \lambda )$ operations to achieve an $\epsilon$-accurate solution (this is easily derived from \cite{AltschulerWR17}, see supplementary material). While the gradient of $\wregla$ does not require further operations, 
the gradient of $\wtildela$  requires 
the inversion of an $n\!\times\! n$ matrix 
as stated in \Cref{alggradient}. 
However, since the matrix to be inverted is a sum of a diagonal matrix and a low-rank matrix, the inversion  requires $O(n m^2)$ operations (e.g. using Woodbury matrix identity), for a total cost of the gradient equal to $O(n m(\epsilon^{-2} \lambda + m))$. 
Even for very large $n$, \Cref{alggradient} is still efficient in all those settings where $m<< n$. 
   \\

Moreover, note that the most expensive additional operations 
consist of matrix multiplications and the inversion of a positive definite matrix, which are very efficiently implemented on modern machines. Indeed, in our experiments the Sinkhorn algorithm was always the most expensive component of the computation. 
It is important to notice however that in practical applications both routines can be parallelized, and several ideas can be exploited to lower the computational costs of either algorithms depending on the problem structure (see for instance the {\em convolutional \wass\ distance} in \cite{solomon2015convolutional}). Therefore, depending on the setting, the computation of the gradient of the sharp Sinkhorn could be comparable or significantly slower than the regularized Sinkhorn or the automatic differentiation considered in \cite{genevay2018learning}. 
\end{remark}


\section{Learning with Sinkhorn Loss Functions}\label{learning}

Given the characterization of smoothness for both Sinkhorn distances, in this section we focus on a specific application: supervised learning with a Sinkhorn loss function. Indeed, the result of \Cref{smooth_wtilde} will allow to characterize the statistical guarantees of an estimator devised for this problem in terms of its universal consistency and learning rates. Supervised learning with the (regularized) Sinkhorn loss was originally considered in \cite{Frogner2015}, where an empirical risk minimization approach was adopted. In this work we take a structured prediction perspective \cite{bakir2007predicting}. This will allow us to study a learning algorithm with strong theoretical guarantees that can be efficiently applied in practice.\\

\noindent {\bf Problem Setting.} Let $\xx$ be an input space and $\yy = \Delta_n$ a set of histograms. As it is standard in supervised learning settings, the goal is to approximate a minimizer of the {\em expected risk}
\eqal{\label{eq:expected-risk}
  \min_{f:\xx\to\yy} ~ \E(f), \qquad\qquad \E(f) = \int_{\xx\times\yy} \loss(f(x),y) ~ d\rho(x,y)
}
given a finite number of training points $(x_i,y_i)_{i=1}^\ell$ independently sampled from the unknown distribution $\rho$ on $\xx\times\yy$. The loss function $\loss:\yy\times\yy\to\R$ measures prediction errors and in our setting corresponds to either $\wtildela$ or $\wregla$.\\

\noindent {\bf Structured Prediction Estimator.} Given a training set $(x_i,y_i)_{i=1}^\ell$, we consider $\fhat:\xx\to\yy$ the structured prediction estimator proposed in \cite{CilibertoRR16}, defined such that
\begin{equation}\label{eq:estimator}
  \hat{f}(x)=\argmin_{y\in \yy}\sum_{i=1}^\ell\alpha_i(x)~\loss(y,y_i)
\end{equation}
for any $x\in\xx$. The weights $\alpha_i(x)$ are learned from the data and can be interpreted as scores suggesting the candidate output distribution $y$ to be close to a specific output distribution $y_i$ observed in training {\em according to the metric} $\loss$. While different learning strategies can be adopted to learn the $\alpha$ scores,  we consider the kernel-based approach in \cite{CilibertoRR16}. In particular, given a positive definite kernel $k:\xx\times\xx\to\R$ \cite{aronszajn1950theory}, we have 
\eqal{\label{eq:alphas}
  \alpha(x) = (\alpha_1(x),\dots,\alpha(x))^\top = (K + \gamma\ell  I)^{-1} K_x
}
where $\gamma>0$ is a regularization parameter while $K\in\R^{\ell \times \ell}$ and $K_x\in\R^n$ are respectively the empirical kernel matrix with entries $K_{ij} = k(x_i,x_j)$ and the evaluation vector with entries $(K_x)_i = k(x,x_i)$, for any $i,j=1,\dots,\ell$.



\begin{remark}[Structured Prediction and Differentiability of Sinkhorn]\label{rem:structured-prediction-sinkhorn}
The current work provides both a {\em theoretical} and {\em practical} contribution to the problem of learning with Sinkhorn distances. On one hand, the smoothness guaranteed by \Cref{smooth_wtilde} will allow us to characterize the generalization properties of the estimator (see below). On the other hand, \Cref{gradientwtilde} provides an efficient approach to {\em solve} the problem in \Cref{eq:estimator}. Indeed note that this optimization corresponds to solving a barycenter problem in the form of \Cref{eq:bary-functional}. Given the gradient estimation algorithm in \Cref{alggradient}, this work allows to solve it by adopting first order methods such as gradient descent.  
\end{remark}

\noindent {\bf Universal Consistency of $\fhat$}. We now characterize the theoretical properties of the estimator introduced in \Cref{eq:estimator}. We start by showing $\fhat$ is {\em universally consistent}, namely that it achieves minimum expected risk as the number $\ell$ of training points increases. To avoid technical issues on the boundary, in the following we will require $\yy = \Delta_n^\epsilon$ for some $\epsilon>0$ to be the set of points $p\in\Delta_n$ with $p_i\geq\epsilon$ for any $i=1,\dots,n$. The main technical step in this context is to show that for any smooth loss function on $\yy$, the estimator in \Cref{eq:estimator} is consistent. In this sense, the characterization of smoothness in \Cref{smooth_wtilde} is key to prove the following result, in combination with Thm. 4 in \cite{CilibertoRR16}.

\begin{restatable}[Universal Consistency]{theorem}{TUniversal}\label{thm:consistency}
Let $\yy=\Delta_n^\epsilon$, $\la>0$ and $\loss$ be either $\wregla$ or $\wtildela$. Let $k$ be a bounded continuous universal\footnote{\noindent This is a standard assumptions for universal consistency (see \cite{steinwart2008support}). Example: $k(x,x') = \textnormal{e}^{-\|x-x'\|^2/\sigma}$.} kernel on $\xx$. For any $\ell \in \N$ and any distribution $\rho$ on $\xx \times \yy$ let $\fhat_\ell:\xx \to \yy$ be the estimator in \Cref{eq:estimator} 
trained with $(x_i,y_i)_{i=1}^\ell$ points independently sampled from $\rho$ and $\gamma_\ell= \ell^{-1/4}$. Then
\eqal{\label{eq:universal-consistency}
  \lim_{\ell\to \infty} {\cal E}(\fhat_\ell) ~~=~   \min_{f:\xx\to\yy} {\cal E}(f) \quad \textrm{with probability} ~ 1.
}
\end{restatable}
The proof of \Cref{thm:consistency} is reported in Appendix \ref{sec:applearning}. A result analogous to the one above was originally proved in \cite{CilibertoRR16} (Thm. 4) for a wide family of functions referred to as {\em Structure Encoding Loss Function (SELF)} (see \cite{ciliberto2017consistent} or the appendix of this work). While several loss functions used in structured prediction have been observed to satisfy this SELF definition, such characterization was not available for the Sinkhorn distances. The main technical step in the proof of \Cref{thm:consistency} in this sense is to prove that any smooth function on $\yy$ satisfies the definition of SELF (see \Cref{def:SELF} and  \Cref{thm:smooth-self} in the Appendix). Combining this result with Thm.$4$ in \cite{CilibertoRR16}, we obtain that {\em for every smooth loss function $\loss$ on $\yy$ the corresponding estimator $\fhat$ in \Cref{eq:estimator} is universally consistent}. The universal consistency of the Sikhorn distances is therefore guaranteed by the smoothness result of \Cref{smooth_wtilde}.

\Cref{thm:consistency} guarantees $\fhat$ to be a valid estimator for the learning problem. To our knowledge, this is the first result characterizing the universal consistency of an estimator minimizing (an approximation to) the \wass\ distance. \\

\noindent{\bf Learning Rates.} By imposing standard regularity conditions on the learning problem, it is possible to provide also excess risk bounds for $\fhat$. Since these conditions are quite technical, we provide here a brief overview while deferring an in-depth discussion to \Cref{sec:applearning}. 

\noindent We start from the observation (see e.g. Lemma $6$ in \cite{CilibertoRR16}) that the solution $\fstar:\xx\to\yy$ of the learning problem introduced in \Cref{eq:expected-risk} is such that
\eqal{
	\fstar(x) = \argmin_{z\in\yy} ~ \int_{\yy} \loss(z,y)~d\rho(y|x)
}
almost surely on $\xx$. In particular $\fstar(x)$ corresponds to the minimizer of the {\em conditional expectation} $\mathbb{E}_{y|x} \loss(z,y)$ of the loss $\loss(z,y)$ with respect to $y$ given $x\in\xx$. As it is standard in statistical learning theory, in order to obtain generalization bounds for estimating $\fstar$ we will impose regularity assumptions on the conditional distribution $\rho(\cdot|x)$ or, more precisely, on its corresponding {\em conditional mean embedding} (\cite{song2009hilbert,song_kernel_2013}) with respect to a suitable space of functions. 

Let $k:\xx\times\xx\to\R$ be the kernel used for the estimation of the weights $\alpha$ in \Cref{eq:alphas} and let $\ff$ be the associated reproducing kernel Hilbert spaces (RKHS) (see \cite{aronszajn1950theory} for a definition). Let $h:\yy\times\yy\to\R$ be the kernel $h(y,y') = \textnormal{e}^{-\|y - y'\|}$ on the output set $\yy$. The RKHS associated to $h$ is $\hh = W_2^{(n+1)/2}(\yy)$, the Sobolev space of square integrable functions with smoothness $\frac{n+1}{2}$ (see e.g. \cite{wendland2004scattered}). We consider a function $\gstar:\xx\to\hh$ such that 
\eqal{\label{eq:gstar}
	\gstar(x) = \int_{\yy} h(y,\cdot) d\rho(y|x)
}
almost surely on $\xx$. For any $x\in\X$, the quantity $\gstar(x)$ is known as the {\em conditional mean embedding} of $\rho(\cdot|x)$ in $\hh$, originally introduced in \cite{song2009hilbert,song_kernel_2013}. In particular, in \cite{song2009hilbert} it was shown that in order to obtain learning rates for an estimator approximating $\gstar$, a key assumption is that $\gstar$ belongs to  $\hh\otimes\ff$, the tensor product between the space $\hh$ on the output and the space $\ff$ on the input. In this work we will require the same assumption. 

\noindent It can be verified that $\hh\otimes\ff$ is a RKHS for vector-valued functions \cite{micchelli2005learning,lever2012conditional,alvarez2012kernels} and that by asking $\gstar\in\hh\otimes\ff$ we are requiring the conditional mean embedding of $\rho(\cdot|x)$ to be sufficiently regular {\em as a function on $\xx$}. We are now ready to report our result on the statistical performance of $\fhat$.

\begin{restatable}[Learning Rates]{theorem}{TRates}\label{thm:rates-formal} Let $\yy=\Delta^\varepsilon_n$, $\lambda>0$ and $\loss$ be either $\wlambda$ or $\wtildela$. Let $\hh = W_2^{(n+1)/2}(\yy)$ and $k:\xx\times \xx\rightarrow \R$ be a bounded continuous reproducing kernel on $\xx$ with associated RKHS $\ff$. Let $\hat{f}_\ell:\xx\rightarrow \yy$ be the estimator in \Cref{eq:estimator} trained with $\ell$ training points independently sampled from $\rho$ and with $\gamma=\ell^{-1/2}$. If $\gstar$ defined in \Cref{eq:gstar} is such that $\gstar\in\hh\otimes\ff$, then
\eqal{\label{eq:learning-rate}
 \E(f) - \min_{f:\xx\to\yy} \E(f) \leq c\,\tau^2 \ell^{-1/4}
 }
 holds with probability $1-8\textnormal{e}^{-\tau}$ for any $\tau>0$, with $c$ a constant independent of $\ell$ and $\tau$.
\end{restatable}
The proof of \Cref{thm:rates-formal} requires to combine our characterization of the Sinkhorn distances (or more generally smooth functions on $\yy$) as structure encoding loss functions (see \Cref{thm:smooth-self}) with Thm. $5$ in \cite{CilibertoRR16} where a result analogous to the one above is reported for SELF loss functions. See \Cref{sec:applearning} for a detailed proof. 


\begin{remark} A relevant question is whether the  \wass\ distance could be similarly framed in the setting of structured prediction. However, the argument used to address Sinkhorn distances relies on their smoothness properties and cannot be extended to the \wass\ distance, which is not differentiable. A completely different approach may still be successful and we will investigate this question in future work.
\end{remark}




\noindent We conclude this section with a note on previous work. We recall that \cite{Frogner2015} has provided the first {\em generalization bounds} for an estimator minimizing the regularized Sinkhorn loss. In \Cref{thm:rates-formal} however we characterize the {\em excess risk bounds} of the estimator in \Cref{eq:estimator}. The two approaches and analysis are based on different assumptions on the problem. Therefore, a comparison of the corresponding learning rates is outside the scope of this work and is left for future research.

\section{Experiments}\label{Exp}

We present here some experiments that compare the two Sinkhorn distances empirically. Optimization was performed with the accelerated gradient from \cite{cuturi14} for $\wtildela$ and Bregman projections \cite{BenamouCCNP15} for $\wlambda$. The computation of barycenters with Bregman iteration is extremely fast compared to gradient descent. We have then used  the barycenter of $\wlambda$ computed with Bregman projections as \textit{initial datum }for gradient descent with $\wtildela$: this has a positive influence on the number of iterations needed to converge and in this light the optimization with respect to the sharp Sinkhorn distance acts as a \textit{refinement}  of the solution with respect to $\wlambda$.  \\ 

\begin{table}
  \caption{Average absolute improvement in terms of the ideal \wass\ barycenter functional $\B_{\was}$ in \Cref{eq:bary-functional} of \textit{sharp} vs \textit{regularized} Sinkhorn, for barycenters of random measures with sparse support.}
  \label{sample-table}
  \footnotesize
   \begin{tabular}{ccccc}
    \multicolumn{5}{c}{\bf{Support}}                   \\
     \bf{Improvement }    & 1\% &  2\% & 10\% & 50\%  \\
    \midrule
    $\B_{\was}(\tilde\mu_\lambda^*)-\B_{\was}(\mu_\la^*)$ & $14.914   \pm  0.076$     & $12.482 \pm 0.135$   & $2.736\pm 0.569$   & $0.258 \pm 0.012$ \\
    \bottomrule
  \end{tabular}
  \label{Table1}
 \end{table}
 
\noindent {\bf Barycenters with Sinkhorn Distances}. We compared the quality of Sinkhorn barycenters in terms of their approximation of the (ideal) \wass\ barycenter. We considered discrete distributions on $100$ bins, corresponding to the integers from $1$ to $100$ and a squared Euclidean cost matrix $M$. We generated datasets of $10$ measures each, where only $k=1,2,10,50$ (randomly chosen) consecutive bins are different from zero, with the non-zero entries sampled uniformly at random between $0$ and $1$ (and then normalized to sum up to $1$). We empirically chose the Sinkhorn regularization parameter $\la$ to be the smallest value such that the output $\Tla$ of the Sinkhorn algorithm would be within $10^{-6}$ from the transport polytope in 1000 iterations. \Cref{Table1} reports the absolute improvement of the barycenter of the sharp Sinkhorn distance with respect to the one obtained with the regularized Sinkhorn, averaged over $10$ independent dataset generation for each support size $k$. As can be noticed, the sharp Sinkhorn consitstently outperforms its regularized counterpart. Interestingly, this improvement is more evident for measures with sparse support and tends to reduce as the support increases. This is in line with the remark in example \Cref{example:bary-delta} and the fact that the regularization term in $\wregla$ tends to encourage oversmoothed solutions.\\

\noindent {\bf Learning with Wasserstein loss.} We evaluated the Sinkhorn distances in an image reconstruction problem similar to the one considered in \cite{KDE} for structured prediction. Given an image depicting a drawing, the goal is to learn how to reconstruct the lower half of the image (output) given the upper half (input). Similarly to \cite{cuturi14} we interpret each (half) image as an histogram with mass corresponding to the gray levels (normalized to sum up to $1$). For all experiments, according to \cite{CilibertoRR16}, we evaluated the performance of the reconstruction in terms of the classification accuracy of an image recognition SVM classifier trained on a separate dataset. To train the structured prediction estimator in \Cref{eq:estimator} we used a Gaussian kernel with bandwith $\sigma$ and regularization parameter $\gamma$ selected by cross-validation.\\
\begin{figure}[t!]
\CenterFloatBoxes
\begin{floatrow}

\hspace*{0.04\textwidth}
\begin{minipage}[t]{0.65\textwidth}
\footnotesize
\begin{tabular}{ccccc}
   & \multicolumn{4}{c}{\bf Reconstruction Error (\%)} \\
    $\#$ \bf Classes& $\boldsymbol{\mathbf{\wtildela}}$  & $\boldsymbol{\mathbf{\wregla}}$ &  Hell\cite{ciliberto2017consistent} & KDE \cite{KDE} \\
   \midrule  
   $\mathbf{2}$  & $\mathbf{ 3.7\pm 0.6}$  &  $4.9\pm 0.9$ & $8.0\pm 2.4$ & $12.0 \pm 4.1$ \\
   $\mathbf{4}$  & $\mathbf{ 22.2 \pm 0.9}$  &  $31.8 \pm 1.1$ & $29.2\pm 0.8$ & $40.8 \pm 4.2$\\
   $\mathbf{10}$ & $\mathbf{ 38.9 \pm 0.9}$  &  $44.9\pm 2.5$ & $48.3 \pm 2.4$ & $64.9 \pm 1.4$ \\
    \bottomrule
  \end{tabular}
\end{minipage}
\begin{minipage}[t]{0.3\textwidth}
\ffigbox{%
\includegraphics[width=0.5\textwidth]{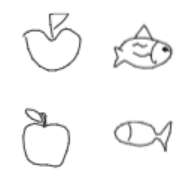}%
}{}
\end{minipage}
\caption{Average reconstruction errors of the Sinkhorn, Hellinger, and KDE estimators on the Google QuickDraw reconstruction problem. On the right, a mini-sample of the dataset. \label{tab:QuickDraw}}
\end{floatrow}
\end{figure} 

\noindent {\emph {Google QuickDraw.}} We compared the performance of the two estimators on a challenging dataset. We selected $c = 2,4,10$ classes from the Google QuickDraw dataset \cite{quickdraw} which consists in images of size $28\times28$ pixels. We trained the structured prediction estimators on $1000$ images per class and tested on other $1000$ images. We repeated these experiments $5$ times, each time randomly sampling a different training and test dataset. \Cref{tab:QuickDraw} reports the reconstruction error (i.e. the classification error of the SVM classifier) over images reconstructed by the Sinkhorn estimators, the structured prediction estimator with Hellinger loss \cite{CilibertoRR16} and the Kernel Dependency Estimator (KDE) \cite{KDE}. As can be noticed, both Sinkhorn estimators perform significantly better than their competitors (except the Hellinger distance outperforming $\wlambda$ on 4 classes). This is in line with the intuition that optimal transport metrics respect the way the mass is distributed on images \cite{sink,cuturi14}. Moreover, it is interesting to note that the estimator of the sharp Sinkhorn distance provides always better reconstructions than its regularized counterpart, supporting the idea that it is more suited to settings where \wass\ distance should be used. \\

The experiments above are a preliminary assessment of the potential of sharp Sinkhorn distance in barycenters and learning settings. More extensive experiments on real data will be matter of future work.

\section{Discussion}

In this paper we investigated the differential properties of Sinkhorn distances. We proved the smoothness of the two functions and derived an explicit algorithm to efficiently compute the gradient of the sharp Sinkhorn distance. Our result allows to employ the sharp Sinkhorn distance in applications that rely on 
first order optimization methods, such as in approximating \wass\ barycenters and supervised learning on probability distributions. In this latter setting, our characterization of smoothness allowed to study the statistical properties of the Sinkhorn distance as loss function. In particular we considered a structured prediction estimator for which we proved universal consistency and generalization bounds.
Future work will focus on further applications and a more extensive comparison with the existing literature.

\bibliographystyle{abbrv}
\bibliography{biblio}

\newpage
\appendix
\section*{\Large Supplementary Material}
\section{Barycenter of Dirac Deltas}\label{Apsec:barydelte}
Wasserstein barycenter problems can be divided into two main classes: problems in which the support is free (and must be computed, generating a nonconvex problem \cite{cuturi14}) and problems where the support is fixed. In some cases, the latter is the only valid choice: for instance, when the geometric domain is a space of symbols and the cost matrix $M$ contains the symbol-to-symbol dissimilarities, no extra information of the symbol space is available and the support of the barycenter will have to lie on a pre-determined set in order to be meaningful. A concrete example is the following: when dealing with histograms on words, the barycenter will optimize how to spread the mass among a set of known words that are used to build the matrix $M$, through a $\mathsf{word2vec}$ operation. In the following we carry out the computation of the barycenter of two Dirac deltas with regularized Sinkhorn and Sinkhorn distances, in order to prove what stated in example \ref{example:bary-delta}.
\paragraph{Barycenter with $\wlam$:}
\noindent Let $\mu=\delta_{z}$ be the Dirac delta centered at $z\in\mathbb{R}^d$  and $\nu=\delta_{y}$  the Dirac delta centered at $y\in\mathbb{R}^d$. We fix the set of admissible support of the barycenter $X=\{x_1,\dots,x_n\}$, where $x_i\in\mathbb{R}^d$ for any $i$. For the sake of simplicity  let us assume that $X$ contains the point $(y+z)/2$. The cost matrices with mutual distances between $z$ and $X$ and $y$ and $X$ will be \[ 
M^z=\{\distance(z,x_i)\}_{i=1}^n\in \mathbb{R}^n, \qquad M^y=\{\distance(y,x_i)\}_{i=1}^n.
\]
Since the support is fixed, only the masses $a=(a_1,\dots,a_n)$ of the barycenter $\tilde{\mu}_{\la}=\sum_{i=1}^n a_i \delta_{x_i}$ are to be computed. Vector $a$ is the  minimizer of the following functional
\[ 
\Delta_n \ni a\longrightarrow {\mathcal{B}_{\wlambda}}(a)=\frac{1}{2} \wlam(a,\delta_{z})+\frac{1}{2} \wlam(a,\delta_{y}).
\]
Note that since Dirac delta has mass 1 concentrated at a point, the transport polytope corresponding to $a$ and a Dirac delta is $\Pi(a,1)$. The elements in $\Pi(a,1)$ are those matrices $T \in\mathbb{R}^{n\times 1}$ such that $T\mathbbm{1}_1=a$ and $T^\top\mathbbm{1}_n=1$. Thus, \begin{align} 
 \begin{pmatrix}
           T_{1} \\
           T_{2} \\
           \vdots \\
           T_{n}
         \end{pmatrix} \begin{pmatrix}
		1
         \end{pmatrix}= \begin{pmatrix}
           a_{1} \\
           a_{2} \\
           \vdots \\
           a_{n}
         \end{pmatrix}
\end{align}
which implies $T_1=a_1,\dots,T_n=a_n$. In this case, $\Pi(a,1)$ contains only one matrix, which coincides with $a^\top$. The distance $\wlam(a,\delta_{z})$ is given by $\langle a^\top, M^z \rangle-\frac{1}{\lambda}h(a)$ and, similarly,   $\wlam(a,\delta_{y})=\langle a^\top, M^y \rangle-\frac{1}{\lambda}h(a)$. Then, the goal is to minimize \[ 
a\longrightarrow  \frac{1}{2}\langle a,M^z\rangle +\frac{1}{2}\langle a,M^y\rangle +\frac{1}{\lambda}\sum_{i=1}^n a_i(\log a_i -1)\]
with the constraint that $a\in \Delta_n$.
The partial derivative with respect to $a_i$ is given by \[
\frac{\partial \mathcal{B}_{\wlambda}}{\partial a_i}= \frac{1}{2}( M^z_i+ M^y_i)+\frac{1}{\lambda}\log a_i
 \]
 Setting it equal to zero, it yields $a_i=\textnormal{e}^{-\lambda( M^z_i + M^y_i)/2}$. The constraint $a\in\Delta_n$ leads to \[ 
 a_i=\frac{\textnormal{e}^{-\lambda( M^z_i + M^y_i)/2}}{\sum_{j=1}^n\textnormal{e}^{-\lambda( M^z_j + M^y_j)/2}}.
 \] Then the barycenter $\tilde{\mu}^*_\la$ has masses $(a_1,\dots,a_n)$ where each $a_i$ is strictly positive, with maximum at the entry corresponding to the point $x_i$ which realizes the minimum  distance from $z$ and $y$, i.e. $(z+y)/2$. The sparsity of the initial deltas is lost. \vspace{1cm}
 
\paragraph{Barycenter with $\wtildela$:} On the other hand, let us compute the barycenter between $\mu$ and $\nu$ with respect to the Sinkhorn distance recalled in \eqref{wtilde}. The very same considerations on $\Pi(a,1)$ still hold, so $\Pi(a,1)$ contains $T=a^\top$ only.
Hence, in this case the Sinkhorn barycenter functional $\mathcal{B}_{\wtildela}$ coincides with the Wasserstein barycenter functional $\mathcal{B}_{\was}$, since  $\wtildela(a,\delta_j)=\langle a^\top,M^j\rangle=\was(a,\delta_j)$, for $j=z,y$. This trivially implies that $\mu_{\la}^*=\mu_{\was}^*$.

\section{Proof of Proposition \ref{convergence_to_wass} in section \ref{sec:motivations}}
\convtowass*
\begin{proof}
As shown in \cite{Cominetti1994}(Prop.5.1), the sequence $\Tla$ converges to an optimal plan of $\was$ as $\lambda$ goes to infinity. More precisely,  \[
\Tla \rightarrow T^*=\textnormal{argmax}_{T\in\Pi(a,b)} \{ h(T); \,\,\, \langle T,M\rangle=\was(\mu,\nu) \} 
  \]
exponentially fast, that is $\norm{\Tla - T^*}_{\mathbb{R}^{nm}}\leq c\,\textnormal{e}^{-\lambda}$.
Thus,\[
\abs{\wtilde_\lambda(\mu,\nu)-\was(\mu,\nu)}=\abs{\langle \Tla,M\rangle-\langle T^*,M\rangle}\leq \norm{\Tla-T^*}\norm{M}\leq  c\,\textnormal{e}^{-\lambda}\norm{M}=:c_1\textnormal{e}^{-\lambda}. 
  \] 
As for the second part, let  $T^*$ be the $\textnormal{argmax}_{T\in\Pi(a,b)} \{ h(T); \,\,\, \langle T,M\rangle=\was(\mu,\nu) \}$. By optimality of $\Tla$ and $T^*$ for their optimization problems, it holds \[ 
0\leq \langle\Tla,M\rangle- \langle T^*,M\rangle\leq\lambda^{-1}(h(\Tla)-h(T^*));
 \]
 Indeed, since $\Tla$ is the optimum, it attains the minimum and hence \[ 
\langle \Tla, M\rangle -\lambda^{-1}h(\Tla)\leq \langle T,M\rangle -\lambda^{-1}h(T)
 \]
 for any other $T$, including $T^*$. 
 By definition of $\wlam$ and $\was$, the inequalities above can be rewritten as
 \[ 
0\leq \wlam(\mu,\nu)-\was(\mu,\nu)\leq \lambda^{-1}h(T^*)=:c_2\lambda^{-1}
 \]
 which goes to $0$ with speed $\lambda^{-1}$ as $\lambda$ goes to infinity.
\end{proof}

\section{Proofs on  differential properties and formula of the gradient}\label{sec:appdiff}
In this section we go over all the details of the proofs sketched in section \ref{sec:differential-properties}.
\cinftysmooth*
\begin{proof}
Let us show the proof for $\wtilde_\lambda$ first. We organize it in three steps:\\
\textit{Step 1. $\wtilde_\lambda$ is smooth when $\Tla$ is:} when considering histograms, $\wtilde_\lambda$ depends on its argument $a$ and $b$ through the optimal coupling $\Tla(a,b)$, being the cost matrix $M$ fixed. Thus, since $\wtilde_\lambda$ is a smooth function of $\Tla$ (being the Frobenius product of $\Tla$ with a constant matrix), showing that $\wtilde_\lambda$ is smooth in $a,b$ amounts to showing that $\Tla$ is smooth.

\noindent \textit{Step 2. $\Tla$ is smooth when $(\ala,\bla)$ is:} 
By Sinkhorn's scaling theorem \cite{sinkhorn1967}, the optimal plan $\Tla$ is characterized as follows 

\begin{equation}\label{eq:appmatrixT}
\Tla=\textnormal{diag}(\textnormal{e}^{\lambda\ala})\textnormal{e}^{-\lambda M}\textnormal{diag}(\textnormal{e}^{\lambda\bla}.)
\end{equation} 
Being the exponential a smooth function, $\Tla(a,b)$ is smooth in $a$ and $b$ if the dual optima $\ala(a,b)$ and $\bla(a,b)$ are. Our goal is then showing smoothness with respect to $a$ and $b$ of the dual optima. 

\noindent \textit{Step 3. $(\ala,\bla)$ is smooth in $a,b$:} this is the most technical part of the proof. First of all, let us stress that one among the $n+m$ rows/columns constraints of $\Pi(a,b)$ is \textit{redundant}: the standard dual problem recalled in \Cref{functL} has an extra dual variable, and this degree of freedom is clear noticing that if $(\alpha,\beta)$ is feasible, than the pair $(\alpha+t\mathbbm{1}_n,\beta -t \mathbbm{1}_m)$ is also feasible. In the following, we get rid of the redundancy  removing one of the dual variables. Hence, let us set \[ 
\mathcal{L}(a,b;\alpha,\beta)=-\alpha^\top\,a-\beta^\top\,\crop{b}+\sum_{i,j=1}^{n,m-1}\,\frac{\textnormal{e}^{-\lambda(M_{ij}-\alpha_i-\beta_j)}}{\lambda},
\]
where $\crop{b}$ corresponds to $b$ with the last element removed.\\
\noindent To avoid cumbersome notation, from now on we denote $x=(a,b)$ and $\gamma=(\alpha,\beta)$. The function $\mathcal{L}$ is smooth and strictly convex in  $\gamma$: hence, for every fixed $x$ in the interior of $ \Delta_n\times \Delta_n$ there exist $\gamma^*(x)$  such that $\mathcal{L}(x;\gamma^*(x))=\min_{\gamma}\mathcal{L}(x;\gamma)$. We now fix $x_0$ and show that $x\mapsto \gamma^*(x)$ is $\textnormal{C}^k$ on a neighbourhood of $x_0$.
Set $\Psi(x;\gamma):=\nabla_{\gamma}\mathcal{L}(x;\gamma)$; the smoothness of $\mathcal{L}$  ensures that $\Psi\in \textnormal{C}^k$. Fix $(x_0;\gamma_0)$ such that $\Psi(x_0;\gamma_0)=0$. Since $\nabla_\gamma \Psi(x;\gamma)=\nabla^2_\gamma \mathcal{L}(x;\gamma)$ and $\mathcal{L}$ is strictly convex, $\nabla_\gamma \Psi(x_0;\gamma_0)$ is invertible. Then, by the implicit function theorem, there exist a subset $U_{x_0}\subset \Delta_n\times \Delta_n$ and a function $\phi:U_{x_0}\rightarrow \Delta_n\times \Delta_n$ such that  \begin{itemize}
\item[i)]$\phi(x_0)=\gamma_0$
\item[ii)]$\Psi(x,\phi(x))=0,\qquad \forall x\in U_{x_0}$
\item[iii)]$\phi \in \textnormal{C}^k(U_{x_0})$.
\end{itemize}
For each $x$ in $U_{x_0}$, since $\phi(x)$ is a stationary point for $\mathcal{L}$ and $\mathcal{L}$ is strictly convex, then $\phi(x)=\gamma^*(x)$, which is- recalling the notation set before- $(\ala,\bla)$. By  a standard covering argument, $(\ala,\bla)$ is $\textnormal{C}^k$ on the interior of $\Delta_n\times \Delta_n$. As this holds true for any $k$, the optima $(\ala,\bla)$, and hence $\wtilde_\lambda$, are $\textnormal{C}^{\infty}$ on the interior of $\Delta_n\times \Delta_n$.\\

\noindent Let us now focus on the smoothness of $\wlambda$. Note that when $a,b$ belong to the interior of the simplex, all components are strictly positive. From the characterization of $\Tla$ recalled in \Cref{eq:appmatrixT}, we know ${\Tla}_{ij}>0$ for any $i,j=1\dots n,m$. Then, since the logarithm is a smooth function of $\Tla$, the term $\lambda^{-1}h(\Tla)$ is smooth in $a$ and $b$. This fact combined with the first part of the proof shows the smoothness of  $\wlambda(a,b)=\langle \Tla, M\rangle -\lambda^{-1} h(\Tla)$.
\end{proof}
With a similar procedure, the implicit function theorem provides a formula for the gradient of sharp Sinkhorn distance.
\gradient*
\begin{proof}
Let us adopt the same notation as in the previous proof. Since $\Psi=\nabla_{(\alpha,\beta)}\mathcal{L}$, by a direct computation, $\Psi$ can be written as \[ 
\Psi(a,b;\alpha,\beta)= 
\begin{pmatrix} 
a-C\mathbbm{1} \\
b-C^\top\mathbbm{1}
\end{pmatrix},
\] 
where $C$ is the $n\times m-1$ matrix given by $\textnormal{diag}(\textnormal{e}^{\lambda\ala})\textnormal{e}^{\lambda\crop{M}}\textnormal{diag}(\textnormal{e}^{\lambda\bla})$ and $\crop{M}$ is the matrix $M$ with $m^{th}$ column removed. In the following, we keep track of the dependence on $a$ only. Being $\Psi$ the gradient of $\mathcal{L}$, and $\gamma^*(a)=(\ala(a),\bla(a))$ a stationary point, we have \begin{equation}\label{startingpoint}
\Psi(a;\gamma^*(a))=0.
\end{equation}
For the sake of clarity, notice that: \begin{itemize}
\item[i)] $a\in \mathbb{R}^n$;
\item[ii)]$\mathcal{L}: \mathbb{R}^{n} \times \mathbb{R}^{n} \times \mathbb{R}^{m-1}\longrightarrow \mathbb{R}$, as we are considering it is a function of $a$, $\alpha$ , $\beta$;
\item[iii)] $\Psi(a,\gamma(a))=\nabla_{\alpha,\beta}\mathcal{L}(a,\gamma(a)) \in \mathbb{R}^{n+m-1 \times 1}$;
\item[iv)]  $\ala:\mathbb{R}^{n}\rightarrow \mathbb{R}^n$, $\bla:\mathbb{R}^{n}\rightarrow \mathbb{R}^{m-1}$, thus $\gamma^*: \mathbb{R}^n\rightarrow \mathbb{R}^n\times \mathbb{R}^{m-1}$.
\end{itemize}
Our goal is to derive $\nabla_a \gamma^*(a)$: by matrix differentiation rules \cite{kollo2006advanced} and \Cref{startingpoint}, 
\begin{equation}\label{grad1}
\nabla_a\Psi(a,\gamma^*(a))=\nabla_1 \Psi(a,\gamma^*(a))+\nabla_a \gamma^*(a)\nabla_2\Psi(a,\gamma^*(a))=0.
\end{equation}

Let us analyse each term: $\nabla_1 \Psi(a,\gamma^*(a))=[\textnormal{I}_{n},\mathbf{0}_{n,m-1}]$ is $n\times n\!+\!m\!-\!1$ matrix with identity and zeros block, and $\nabla_2 \Psi(a,\gamma^*(a))=\nabla^2_{\alpha,\beta}\mathcal{L}(a,\gamma^*(a))=:H$ is the  Hessian of $\mathcal{L}$ evaluated at $(a,\gamma^*(a))$, which is a $n\!+\!m\!-\!1\times n\!+\!m\!-\!1$ matrix.
Together with \Cref{grad1}, this yields \[ 
\nabla_a \gamma^*(a)=[\nabla_a\ala(a),\nabla_a \bla(a)]=-D H^{-1}.
\]
For the sake of clarity, note that $\nabla_a\ala(a)$ and $\nabla_a \bla(a)$ contains the gradients of the components as columns, i.e. \begin{align*}
&\nabla_a \ala=
\begin{pmatrix} 
\nabla_a {\ala}_1, & \nabla_a {\ala}_2, & \dots, & \nabla_a {\ala}_n
\end{pmatrix}\\
&\nabla_a \bla=
\begin{pmatrix} 
\nabla_a {\bla}_1, & \nabla_a {\bla}_2, & \dots, & \nabla_a {\bla}_{m-1}
\end{pmatrix}.
\end{align*}
Now, since $\wtilde_\lambda(a,b)=\langle \Tla,M\rangle$ and $\Tla$ corresponds to \Cref{eq:appmatrixT} a straightforward computation shows that \[
\nabla_a \wtilde_\lambda(a,b)=\sum_{i,j=1}^{n,m}\nabla_a {\Tla}_{ij} M_{ij}=\lambda\sum_{i,j=1}^{n,m} {\Tla}_{ij}M_{ij}\nabla_a {\ala}_i+\lambda\sum_{i,j=1}^{n,m-1} {\Tla}_{ij}M_{ij}\nabla_a {\bla}_j.
 \]
 Setting $L:=\Tla\odot M$, then the formula above can be rewritten as \[ 
 \nabla_a \wtilde_\lambda(a,b)=\lambda\sum_{i}^{n}\nabla_a{\ala}_i\sum_{j=1}^{m}L_{ij}+\lambda\sum_{j=1}^{m-1}\nabla_a {\bla}_j\sum_{i=1}^{n}L_{ij},
 \]
 which is exactly \[ 
 \nabla_a \wtilde_\lambda(a,b)=\lambda(\nabla_a\ala L \mathbbm{1}_m +\nabla_a\bla \crop{L}^\top \mathbbm{1}_n).
 \]
Since by definition, the gradient belongs to the tangent space of the domain, and $a\in \Delta_n$, we project on the tangent space of the simplex, recovering $\proj_{\textnormal{T}\Delta_n}\lambda(\nabla_a\ala L \mathbbm{1}_m +\nabla_a\bla \crop{L}^\top \mathbbm{1}_n).$
\end{proof}

\subsection{Massaging the gradient to get an algorithmic-friendly form}
In the proof of theorem  \ref{gradientwtilde} we have derived a formula for the gradient of Sinkhorn distance. In this section we further manipulate it in order to obtain an algorithmic friendly expression that also points out some interesting bits that were hidden in the formula above. All the notation has already been introduced: from now on, we will drop the $\lambda$ and denote the optimal plan by $T$ to make the notation neater.\\

\noindent An explicit computation of the second derivatives of $\mathcal{L}$ with respect to $\alpha_i$ and $\beta_j$ for $i=1,\dots,n$ and $j=1,\dots,m-1$ leads to the following identity \[
H=\begin{pmatrix} 
\textnormal{diag}(T\mathbbm{1}) & \crop{T} \\
\crop{T}^{\top} &\textnormal{diag}(\crop{T}^{\top}\mathbbm{1}) 
\end{pmatrix}=\begin{pmatrix} 
\textnormal{diag}(a) & \crop{T} \\
\crop{T}^{\top} &\textnormal{diag}(\crop{b}) 
\end{pmatrix}.
 \] 
 That is, $H$ is a block matrix and each block can be expressed in terms of the plan $T$. The block structure can be exploited when it comes to compute the inverse: we have shown that the gradient of the dual potentials is given by \[ 
[\nabla_a\ala,\nabla_a\bla]=-DH^{-1}, \qquad D=[\textnormal{I}_{n}, \mathbf{0}_{n,m-1}]. 
 \]
 Now,  the inverse of a block matrix is again a block matrix, say \[
 H^{-1}=\begin{pmatrix} 
K_1 & K_2\\
K_3 & K_4
\end{pmatrix}.
  \] 
Then, $[\nabla_a\ala,\nabla_a\bla]=-[K_1, K_2]$. By the formula of the block inverse, setting \[
\mathcal{K}= \textnormal{diag}(T\mathbbm{1}) -  \crop{T} \textnormal{diag}(\crop{T}^{\top}\mathbbm{1})^{-1} \crop{T}^{\top},
\] 
the blocks $K_1$ and $K_2$ are given by
\[
K_1= \mathcal{K}^{-1}, \qquad K_2= -\mathcal{K}^{-1} \crop{T} \textnormal{diag}({\crop{T}}^{\top}\mathbbm{1})^{-1}.
\]
Note that $\mathcal{K}$ is symmetric and so its inverse. Now, we can rewrite $\lambda (\nabla_a\ala L \mathbbm{1}_m +\nabla_a\bla \crop{L}^\top \mathbbm{1}_n)$, with $L=T\odot M$, as \[ 
\lambda\big(-\mathcal{K}^{-1}S \mathbbm{1}_m+\mathcal{K}^{-1} \crop{T} \textnormal{diag}(\crop{T}^{\top}\mathbbm{1})^{-1}\crop{L}^\top \mathbbm{1}_n\big)
\]
and we  conclude that 
\[ 
\nabla_a\wtilde_\lambda(a,b)=\lambda \cdot\textnormal{solve}(\mathcal{K}, -L \mathbbm{1}_m +  \crop{T} \textnormal{diag}(\crop{T}^{\top}\mathbbm{1})^{-1}\crop{L}^\top \mathbbm{1}_n ).\]

\paragraph{Comment on Remark \ref{rem:computation-comparison}:} In the recent work \cite{AltschulerWR17}, it has been shown that Sinkhorn-Knopp algorithm outputs a matrix $\Tla$ whose distance $\norm{\Tla\mathbbm{1}-a}_1+\norm{\Tla^\top\mathbbm{1}-b}_1$ from the transport polytope $\Pi(a,b)$ is smaller than $\epsilon$ in $O(\epsilon^{-2}\log(s/\ell))$ iterations, where $s=\sum_{ij}\textnormal{e}^{-\lambda M_{ij}}$ and $\ell=\min_{ij}\textnormal{e}^{-\lambda M_{ij}}$. Let us denote by $M_{\textnormal{max}}$ and $M_{\textnormal{min}}$ the maximum and minimum elements of $M$ respectively. Then, \[ 
\frac{s}{\ell}=\sum_{ij}\textnormal{e}^{-\lambda(M_{ij}-M_{\textnormal{max}})}\geq \textnormal{e}^{-\lambda(M_{\textnormal{min}}-M_{\textnormal{max}})}\geq 1.
\]
This yields the lower bound \[ 
\log\Big(\frac{s}{\ell}\Big)\geq c \lambda 
\]
where $c$ is a constant independent of $\lambda$. 
We can then conclude that Sinkhorn-Knopp algorithm returns a matrix $\Tla$ such that \[ 
\langle \Tla,M\rangle \leq \was(a,b) + \epsilon
\]
in $O(n m\epsilon^{-2} M_{\textnormal{max}}^2\lambda )$.

\section{Proofs in \ref{learning}: Learning with Sinkhorn Loss Functions}\label{sec:applearning}

We recall the main definition and tools from \cite{CilibertoRR16} needed to fully understand what discussed in section \ref{learning}. 
The structured prediction estimator recalled in \Cref{eq:estimator} is derived in \cite{CilibertoRR16} for a large class of loss functions $\loss:\yy\times \yy\rightarrow \R$ that are referred to as \textit{Structure Encoding Loss Functions} (SELF) and satisfy the following assumption: 
\begin{definition}[SELF]\label{def:SELF}Let $\yy$ be a set. A function $\loss:\yy \times \yy \rightarrow \R$ is a \emph{Structure Encoding Loss Function} (SELF) if there exists a separable Hilbert space $\hy$ with inner product $\langle \cdot, \cdot\rangle_{\hy}$, a continuous map $\psi:\yy\rightarrow \hy$ and a bounded linear operator $V:\hy \rightarrow \hy$ such that \eqal{\label{SELF}
\loss(y,y^\prime)=\langle \psi(y),V\psi(y^\prime)\rangle_{\hy} \qquad y,\,y^\prime\in \hy.}
\end{definition}
\noindent While in \cite{CilibertoRR16} it has been observed that a wide range of commonly used loss functions are SELF, no such result was known  for Sinkhorn loss. This work also provides an answer to this question. In this direction, let us show a first result on smooth function, which will be a key tool in the rest of the analysis. Note that we will use the  notation $H^{r}$ for the Sobolev space $W_2^{r}$.  

\begin{theorem}(Smooth functions are SELF)\label{thm:smooth-self}
Let $\yy= \Delta_n$. Any function $\loss:\yy\times \yy \rightarrow \R$ such that $\loss \in\textnormal{C}^{\infty}(\yy\times \yy)$ is SELF.
\end{theorem}

\begin{proof}
By assumption $\loss\in\textnormal{C}^{\infty}(\yy\times \yy)$. Since $\yy$ is compact, \eqal{
\textnormal{C}^\infty(\yy\times \yy)=\textnormal{C}^\infty(\yy) \otimes \textnormal{C}^\infty(\yy) \subset H^r(\yy)\otimes H^r(\yy),
}
for $r=(n+1)/2$.
The Sobolev space $H^r(\yy)$ is a Reproducing Kernel Hilbert Space (RKHS) \cite{berlinet2011reproducing} and we denote by $\mathsf{k}_y=\mathsf{k}(y,\cdot)\in H^r(\yy)$ the reproducing kernel. The product space $H^r \otimes H^r$ is also an RKHS with reproducing kernel $\mathsf{K}((y_1,y_2),(y_1^\prime,y_2^\prime))=\mathsf{k}(y_1,y_1^\prime)\mathsf{k}(y_2,y_2^\prime)$, i.e. in general $\mathsf{K}_{y,y^\prime}=\mathsf{k}_y\otimes \mathsf{k}_{y^\prime}$. 
Since $\loss \in H^r \otimes H^r$, by reproducing property there exists a function $V\in H^r \otimes H^r$ such that \eqals{
\loss(y,y^\prime)=\langle V,\mathsf{k}_y\otimes \mathsf{k}_{y^\prime}\rangle_{H^s\otimes H^s}.
}
By the isometric isomorphism $H^r \otimes H^r \cong \textnormal{HS}(H^r,H^r)$ \cite{moretti2013spectral}, with $\textnormal{HS}(H^r,H^r)$ the space of Hilbert-Schmidt operators from $H^r$ to itself, it holds \eqal{
\loss(y,y^\prime)=\langle V,\mathsf{k}_y\otimes \mathsf{k}_{y^\prime}\rangle_{H^r\otimes H^r}=\langle V, \mathsf{k}_y\otimes \mathsf{k}_{y^\prime}\rangle_{\textnormal{HS}}=\textnormal{Tr}(V^*\mathsf{k}_y\otimes \mathsf{k}_{y^\prime})=\langle \mathsf{k}_{y^\prime},V^*\mathsf{k}_y\rangle_{H^r},  
}
where $V^*$ is the adjoint operator of $V$. To meet the conditions of definition \ref{def:SELF} it remains to show that $V^*$ and $\mathsf{k}_y$ are bounded. But $\mathsf{k}_y$ is bounded in $H^r$ for any $y\in \yy$ by definition of reproducing kernel and the operator norm $\norm{V^*}$ is bounded from above by the Hilbert-Schmidt norm $\norm{V}_{\textnormal{HS}}$ which is trivially bounded since $V\in \textnormal{HS}(H^r,H^r)$. 
\end{proof}

\begin{corollary}
The regularized and sharp Sinkhorn losses $\wlambda$ and $\wtildela:\Delta^\epsilon_n\times \Delta^\epsilon_n\rightarrow \R$ are SELF.
\end{corollary}
\begin{proof}
Since $\Delta^\epsilon_n\subset \Delta_n$ is compact and $\wlambda$, $\wtildela$ are $\textnormal{C}^\infty$ in the interior on $\Delta_n\times \Delta_n$ by Thm. \ref{smooth_wtilde}, a direct application of the result above shows that $\wlambda$ and $\wtildela$ are SELF. 
\end{proof}

\noindent Summing up these elements, the proof of Thm. \ref{thm:consistency} easily follows: 
\TUniversal*
\begin{proof}
Since $\wlambda$, $\wtildela$ are SELF function and $\Delta^\epsilon_n$ is compact, the result follows from Thm. 4 in \cite{CilibertoRR16}.
\end{proof}

We conclude the section with some comments on  Thm. \ref{thm:rates-formal} and its proof.
We have shown that $\wlambda$ and $\wtildela$ are SELF and can be written as  \eqal{
\wtildela(y,y^\prime)=\langle \mathsf{k}_y,V\mathsf{k}_{y^\prime}\rangle_{H^r(\Delta^\epsilon_n)}}
with $\mathsf{k}$ the reproducing kernel of the Sobolev space $H^r(\Delta^\epsilon_n)$. 
\TRates*
\begin{proof}
The proof substantially takes advantage of the fact that $\wlambda$ and $\wtildela$ are SELF and inherits the generalization bounds proved in Thm. 5 in \cite{CilibertoRR16}. 
\end{proof}

\begin{remark} A relevant question is whether the  \wass\ distance could be similarly framed in the setting of structured prediction. However, the argument used to address Sinkhorn distances relies on their smoothness properties and cannot be extended to the \wass\ distance, which is not differentiable. A completely different approach may still be successful and we will investigate this question in future work.
\end{remark}

\section{Experiment on MNIST}
This last section is a short supplement to section \ref{Exp}. We present a small experiment on the MNIST dataset that has the same flavour as the experiment on GoogleQuickDraw dataset but adresses a more specific target: to evaluate better the quality of the prediction  rather than the overall quality of the reconstructed image, we train the SVM classifier trained on a separate dataset made of $2000$ examples of lower halves of digits $1,2,5,8,9$. Since the classificator is trained on lower halves only, we have selected a subset of digits with clearly diverse shapes, to disregard any legitimate vagueness. This means that any classification errors will be due to a poor prediction of the lower half.

We performed the reconstruction with both $\wlambda$ and $\wtildela$ loss. We tested the performance of the two estimators on $100$ examples. \Cref{TabMnist} reports the {\em performance} of the two estimators, as follows:\begin{itemize}
\item[i)] the terms on the diagonal presents the number of misclassification of the lower half predicted with $\wlambda$ and $\wtildela$ losses;
\item[ii)] the number on the upper diagonal represents the number of errors occurred in the classification of the prediction with $\wlambda$ on those examples that were correctly classified when reconstructed with $\wtildela$;

\item[iii)] conversely, the number on the lower diagonal represents the number of errors occurred in the prediction with $\wtildela$ on those examples that were correctly classified when reconstructed with $\wlambda$.
\end{itemize} 
To be more precise, denote by $\textnormal{L}(\wlambda)$ the vector with labels predicted by the classifier when tested on the halves of digits predicted with $\wlambda$ loss and analogously $\textnormal{L}(\wtildela)$ the vector with labels given by the classifier tested on the halves of images predicted with $\wtildela$ loss. Vector $\textnormal{L}$ is the vector with the true labels of the test set.
Consider two vectors $\tilde{e}^\lambda\in \{0,1\}^{100}$ and $e^\lambda\in \{0,1\}^{100}$ defined as follows: \[
{\tilde{e}^\lambda}_i=\begin{cases}
0 \qquad&\textnormal{if } L_i=\textnormal{L}(\wlambda)_i\\
1 &\textnormal{otherwise}
\end{cases} \qquad {e^\lambda}_i=\begin{cases}
0 \qquad&\textnormal{if } L_i=\textnormal{L}(\wtildela)_i\\
1 &\textnormal{otherwise}.
\end{cases}
\]
Table in Fig. \ref{TabMnist} corresponds to  \[
\begin{pmatrix}
\sum_i  \tilde{e}^{\lambda}_i & \sum_i \tilde{e}^{\lambda}_i (1-e^{\lambda}_i) \\[10pt]
\sum_i e^{\lambda}_i (1-\tilde{e}^{\lambda}_i) & \sum_i  e^{\lambda}_i.
\end{pmatrix}.
\]
What we observed is the following: since the classifier was trained and tested on the lower halves only, the blurriness in the reconstruction performed with $\wlambda$ played a substantial role in the misclassification on digit $5$ in favour of digit $8$.
On the other hand, the sharpness of the reconstruction with $\wtildela$ is a major advantage for the correct classification.
\begin{figure}[t]
\centering
\CenterFloatBoxes
\begin{floatrow}
\begin{minipage}[t]{0.5\textwidth}
\hspace*{0.01\textwidth}
    \begin{tabular}{lll}
      $\#$err  & $\wlambda$      & $\wtilde_\lambda$ \\
    \midrule
   $\wlambda$ & $16$  &  $11$     \\
   $\wtilde_\lambda$ &   $1$  &  $6$\\
    \bottomrule
  \end{tabular}

\end{minipage}
\begin{minipage}[t]{0.6\textwidth}
\ffigbox{%
\hspace*{-0.5\textwidth}
\includegraphics[width=0.6\textwidth]{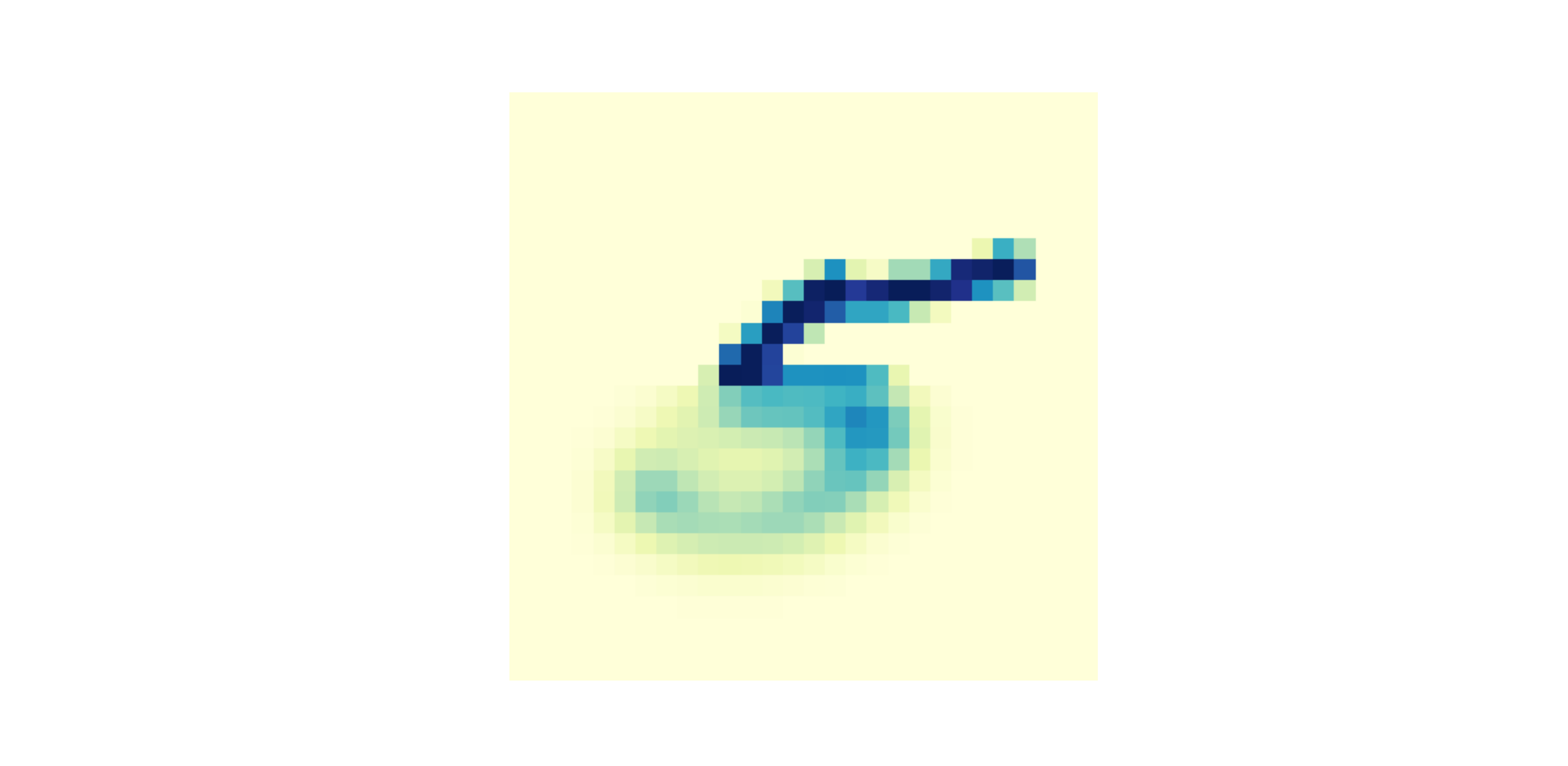}%
\includegraphics[width=0.6\textwidth]{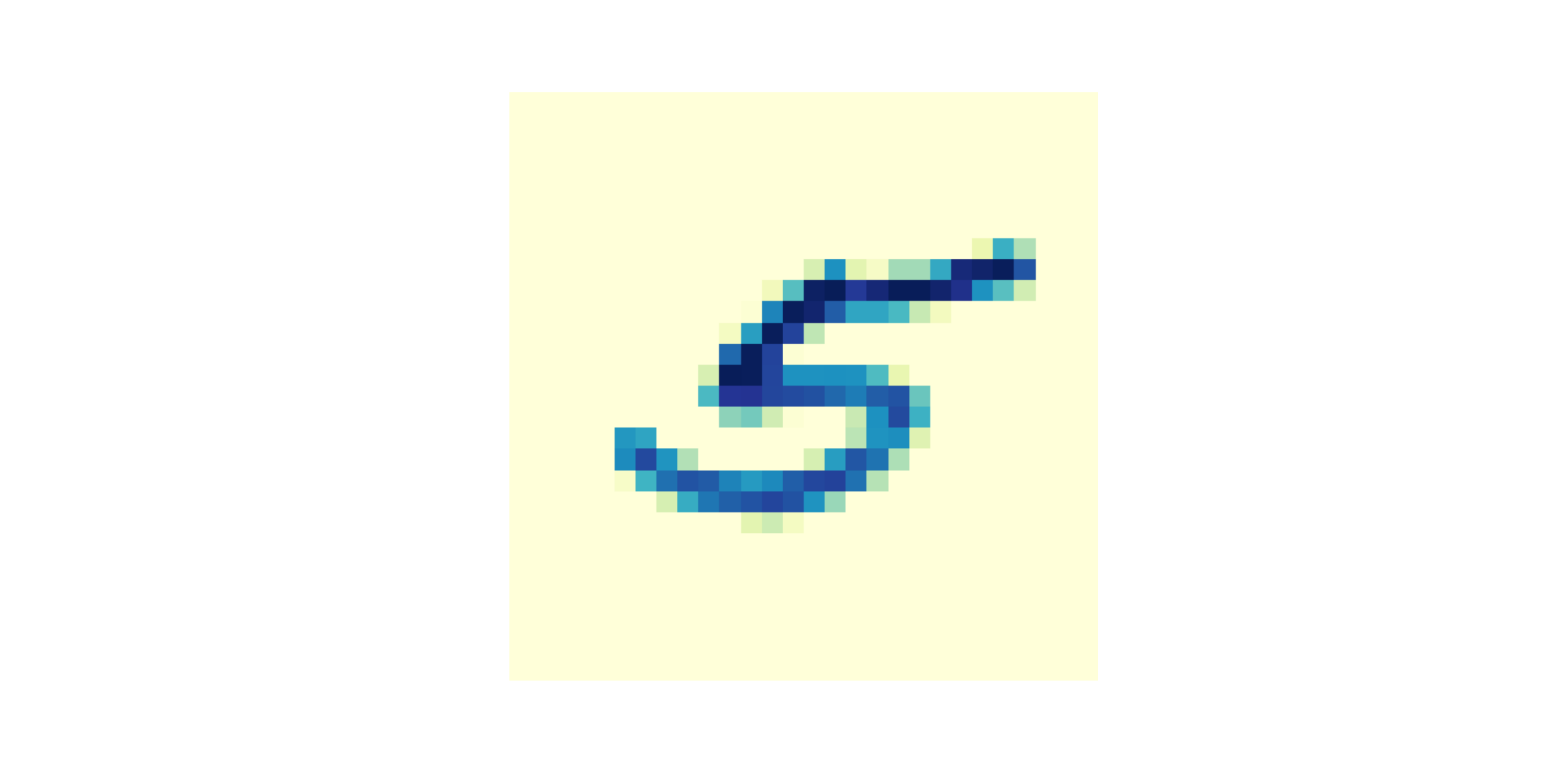}
}{}
\end{minipage}
\caption{(Right) Relative error (see text) for the Sinkhorn estimators on the digit reconstruction problem. (Left) Sample predictions for egularized (First image) and sharp Sinkhorn estimators. \label{TabMnist}}
\end{floatrow}
\end{figure}
\end{document}